\documentclass[twoside,11pt]{article}

\usepackage{jmlr2e}
\usepackage{graphicx}
\usepackage{tikz}
\usepackage{amsmath}
\usepackage{color}
\usepackage{amssymb}
\usepackage{algorithm}
\usepackage[noend]{algpseudocode}
\usepackage{subcaption}
\usepackage{soul}

\newcommand{\set}[1]{\{#1\}}
\newcommand{\ex}[1]{\exp \left\{#1\right\}}
\newcommand{\avg}[1]{\left<#1\right>}
\newcommand{\condavg}[2]{\left<#1 \big| #2\right>}
\newcommand{\avgb}[1]{\Bigg<#1 \Bigg >}
\newcommand{\prob}[1]{P\left(#1\right)}
\newcommand{\argmax}[2]{\operatorname*{\text{argmax}}_{#1} \left\{#2\right\}}

\newcommand{\tpr}{\text{TPR}}
\newcommand{\fpr}{\text{FPR}}

\newcommand{\R}{\mathbb{R}}
\newcommand{\norm}[1]{\left\lVert#1\right\rVert}

\begin{document}
\title{Unsupervised Evaluation and Weighted Aggregation of Ranked
Predictions}

\author{\name Mehmet Eren Ahsen$^\dagger$ \email  mehmet.erenahsen@mssm.edu\\
\name Robert M Vogel$^{\ddagger \dagger}$ \email r.vogel@ibm.com \\
\name Gustavo A Stolovitzky$^{\ddagger \dagger}$ \email gustavo@us.ibm.com\\
\addr $^\dagger$Icahn School of Medicine at Mount Sinai\\
    Department of Genetics and Genomic Sciences\\
    One Gustave Levy Place, Box 1498\\
    New York, NY , USA\\
\addr $^\ddagger$IBM T.J. Watson Research Center\\
   1101 Kitchawan Road, Route 134,\\
   Yorktown Heights, N.Y., 10598, USA.\\
}

\editor{}

\maketitle

\begin{abstract}
Learning algorithms that aggregate predictions from an ensemble of diverse base classifiers consistently outperform individual methods.  Many of these strategies have been developed in a supervised setting, where the accuracy of each base classifier can be empirically measured and this information is incorporated in the training process.  However, the reliance on labeled data precludes the application of ensemble methods to many real world problems where labeled data has not been curated.  To this end we developed a new theoretical framework for binary classification, the Strategy for Unsupervised Multiple Method Aggregation (SUMMA), to estimate the performances of base classifiers and an optimal  strategy for ensemble learning from unlabeled data.
\end{abstract}



\section{Introduction}

It has long been appreciated that combinations of independent and weak learning methods, in both classification and regression tasks, can be used to make a single strong learning method.
In their work, \citep{dietterich2000ensemble} found this phenomena to be universal for three reasons.  The first reason is statistical, in that real world data is often insufficiently large to uniquely infer a single model.  The second reason is computational limitations, as several commonly used models are non-convex, e.g. neural networks, and consequently often result in a parameter sets that are locally as opposed to globally optimal.  The third and last reason is representational: any individual model may be insufficiently complex to represent all the trends in the data.  In all such cases, combining weakly predictive methods produces a model with lower prediction errors and increased generalizability.

Intuitively, boosting the performance of a learning algorithm by combining weak learners is attributable to aggregating diverse predictions.  Specifically, when each model is independent and better than random their prediction errors average out \citep{dietterich2000ensemble,IEEE2006Polikar,zhou2012ensemble} when combined.

Strategies to train an ensemble learning method can be loosely divided into two categories.  The first relies on training several instances of a model on various splits of data.  Notable examples include random splits of the training data by bootstrap aggregating (bagging) \citep{breiman1996bagging} and carefully choosing data subsets in boosting \citep{schapire1990strength,freund1995desicion}.  The second group, however, solely focuses on optimizing the combination method of base learners subject to some \emph{a priori} assumptions.  By separating the development of an ensemble learner into two tasks, first training the base predictors and second the combination method, we can easily incorporate the predictions from pre-trained and heterogeneous methods.

The optimal combination strategies for multi-classification tasks rely upon the \emph{a priori} assumptions or constraints and the readily available data.  To demonstrate this \citep{kuncheva2014weighted} developed a probabilistic framework to define optimal combination rules under common scenarios.  For simplicity, lets only consider binary classification problems and the two corresponding scenarios.  If each base classifier has identical accuracies, the optimal combination strategy is simple majority vote.  In contrast, when individual method performances vary and are known, a weighted majority vote is optimal \citep{shapley1984optimizing}.  However, in the absence of labeled test data to quantify the performance of methods, it is not obvious how an optimal combination strategy can be applied.

Applying trained models in an unsupervised setting is not unusual.  Indeed, in transfer learning researchers often apply pre-trained models to make inferences from data that are not necessarily representative of the original training data set.  This strategy is a result of the increasing success of complex models in which sufficient computational resources, or the abundance of labeled data for rigorous training are often unavailable.  In this setting it is not obvious how the performance of each pre-trained model on the training data is applicable to the new use case, and consequently the optimal combination strategy for a heterogeneous ensemble method is unknown.

Intuitively, the absence of labeled test data precludes the empirical estimation of each base classifier's performance, making the simple majority vote one of the few possible combination methods.  While true in general, \citep{AppStats1979Dawid} circumvented this challenge by inferring the performance of each base classifier and the true class labels together.  Specifically, they recast their problem so that a maximum likelihood solution could be estimated by the Expectation Maximization (EM) algorithm \citep{dempster1977maximum}.  While an elegant solution to the missing labeled data problem, it suffers from an obvious limitation of the EM algorithm for non-convex optimization problems.  That is the inferred solution was heavily dependent on the initial parameters for optimization.

To solve the initial parameter problem of \citep{AppStats1979Dawid}, \citep{parisi2014ranking} developed the Spectral Meta-Learner (SML).  The authors found that the off-diagonal elements of the covariance matrix of binary predictions, here $\{-1, 1\}$, are related to the balanced accuracies of each base classifier.  With this information they developed SML by linearly approximating the maximum likelihood estimator of \citep{AppStats1979Dawid}.  While SML was shown to produce better seed parameters than majority vote for EM, the method is limited to base classifiers that produce binary predictions.

In this manuscript we developed a new weighted combination strategy for binary classification when class labels are not known.  Uniquely, our method uses rank predictions of base classifiers, which makes SUMMA agnostic to diverse scales of each base classifier, e.g. SVM and logistic regression for classification.  We show that, under an assumption of conditionally independent classifier rank predictions: i) we can reliably estimate the performances of each individual classifier in terms of AUROC (Area Under the Receiver Operating Characteristics), a commonly used performance measure in medical and machine learning literature; ii) we derive an approximate maximum likelihood estimator which yields an ensemble learner whose weights are proportional to the AUROC of each individual base classifier. We call our method the Strategy for Unsupervised Multiple Method Aggregation, or SUMMA.


\section{Theory}
\subsection{Problem Setup}
Consider a data set $\mathcal{D} = \set{(x_k, \sigma_k)}_{k=1}^N$ in which each sample $k$ consists of input data $x_k \in \mathcal{X}$ and belongs to an unknown class denoted by $\sigma_k \in \set{0, 1}$, where by convention $0$ denotes the negative class and $1$ denotes the positive class.
The set $\mathcal{X}$ denotes the feature space which could represent the abundance of RNA transcripts or protein molecules in biological data; or pixel intensities in image data; or stock prices of companies in finance problems.

Let $\set{g_i}_{i=1}^M$ represent an ensemble of classifiers, where each classifier, $g_i: \mathcal{X} \rightarrow \mathbb{R}$, is a mapping from the feature space to real numbers.  The output of each classifier is a real number which can be interpreted as a measure of the relative confidence that the sample belongs one of the classes, which, without loss of generality, will be assumed to be class 1.  For example, in a Bayesian framework the output of a classifier is the posterior probability that a sample belongs to the positive class. In the case of SVM, the output is the distance to the separating hyperplane. Recent research suggest appropriate calibration of the classifier outputs will boost the performance of the ensemble classifier \citep{whalen2013comparative,bella2013effect}.
In the current manuscript we will use the rank transformation as a calibration tool.  Transforming to sample rank encodes the ensemble of predictions by disparate methods to an identical scale, which in consequence, precludes complications introduced by {\it ad hoc} normalization strategies employed in other ensemble strategies \citep{bella2013effect}.

Moreover, ranking of the samples are implicitly done in popular evaluation metrics such as the AUROC \citep{marzban2004roc}. Converting scores to rank space has additional theoretical benefits that we develop over the subsequent sections. As a convention, in this manuscript we assume that samples more likely to belong to the positive class have higher scores as such when ranked in descending order they will appear with lower rank than the samples more likely to belong to the negative class.
Moreover, throughout the text we use the notation $g_i$ if the classifier output is confidence levels, and $f_i$ for the rank transformed output of the base classifiers.

\subsection{Performance Metric}

Consider the $i^{th}$ member of the ensemble of classifiers.  Given a fixed number of samples $N$ and prevalence $\rho$, its task is to assign each sample $k$ a rank $r$ with probability $P(f_i(x_k) = r|\sigma_k)$; for simplicity, we will refer to this quantity as $P_i(r | \sigma_k)$.  To measure the performance of each of the $M$ classifiers we compute the difference between the average sample ranks conditioned on their respective class which we denote by $\Delta$.  Mathematically, $\Delta$ is defined as follows.
\begin{definition}\label{def:delta}
  The performance of the $i^{th}$ classifier is measured by,
  \begin{align}
    \Delta_i &= \condavg{r}{\sigma=0}_i - \condavg{r}{\sigma=1}_i,\nonumber
  \end{align}
  where $\condavg{r}{\sigma=k}_i$, for $k=0,1$, represents the average rank given the respective class for the $i^{th}$ method.
\end{definition}

In Definition~\ref{def:delta}, $\Delta_i$ effectively categorizes methods as being either: random, $\Delta_i$=0; or informative $|\Delta_i| > 0$.  Intuitively, random methods are those unable to rank samples according to the latent class, a consequence of the fact that $P_i(r | \sigma=1) = P_i(r|\sigma=0) = \mathcal{U}(1, N)$, where $\mathcal{U}(1, N)$ denotes the uniform distribution on the set of integers $\{1,2,...,N\}$.
Informative methods, on the other hand,
discriminate rank assignments by the sample class.
As an example, consider two methods $f_i$ and $f_j$ whose performances are related by $\Delta_i = - \Delta_j > 0$ and $|\Delta_i| = |\Delta_j|$.  Each method equivalently utilizes signal for scoring samples, $|\Delta_i|=|\Delta_j|$, however each method has chosen opposing conventions for the ranks corresponding to each latent class.
Generally, this results from  systematic errors such as using the opposite convention for class labels. While $\Delta$ provides an intuitive measure of method performance, we show that it is closely associated to the canonical measure of performance, AUROC \citep{marzban2004roc}.


\begin{theorem} \label{thm:auc}
Given a ranked list of predictions and the corresponding sample class, $\set{(r_k, \sigma_k)}_{k=1}^N$, where $r_k$ is the rank assigned by the method to the sample $k$ and $\sigma_k$ is the true class of sample $k$,
\begin{equation} \label{auc:delta}
  AUROC = \frac{\Delta}{N}+\frac{1}{2},
\end{equation}
where $AUROC$ is estimated using the rectangle rule.
\end{theorem}
\begin{proof}
Let $N$ denote the total number of samples, $N_1$ denote the positive samples and $N_0 = N - N_1$ denote the negative samples. The receiver operating characteristic (ROC) consists of points representing the False Positive Rates (FPR$_i$) and the True Positive Rates (TPR$_i$) empirically evaluated for threshold $i \in \{1,...,N\}$.  The area under the ROC (AUROC) can be estimated according to the rectangle rule:
\begin{equation} \label{AUC:def}
  AUROC = \sum_{i=0}^{N-1} \tpr_i \left(\fpr_{i+1}-\fpr_{i}\right),
\end{equation}
where $FPR_0:=0$ and $TPR_0:=0$.
The elements of the sum behave as follows.  If the $i^{th}$ ranked sample has a positive label then the $\fpr_{i+1}=\fpr_{i}$, and if it has a negative label then $FPR_{i+1} = FPR_{i} + 1/N_0$.  Moreover, if we let $\{r_{i_1},\cdots,r_{i_{N_0}}\}$ denote the ranks of negative samples, then for any threshold $i_l$ the number of true positives is given as $i_l-l$ and consequently $TPR_l= (i_l-l)/N_1$.  Using these observations Equation \eqref{AUC:def} becomes
\begin{eqnarray}
  AUROC &=& \sum_{i=1}^N \text{TPR}_i
							(\text{FPR}_{i+1}-\text{FPR}_{i}) \nonumber \\
      	&=&	\sum_{l : \sigma_l = 0} \frac{\text{TPR}_l}{N_1}\frac{1}{N_0} =
        		\sum_{l = 1}^{N_0}\frac{(r_{i_l}-l)}{N_1}\frac{1}{N_0}
        =\frac{\condavg{r}{\sigma = 0}}{N_1}-\frac{N_0+1}{2N_1}. \label{Delta:eq0}
\end{eqnarray}
Next we express $\Delta$ in terms of $\condavg{r}{\sigma = 0}$:
\begin{eqnarray} \label{eq:delta_cond}
	\frac{\Delta}{N} &=& \frac{\condavg{r}{\sigma = 0}-
			\condavg{r}{\sigma = 1}}{N}
		=
			\frac{\frac{\sum_{i : \sigma_i = 0}r_i}{N_0}-
			\frac{\sum_{i : \sigma_i = 1}r_i}{N_1}
			}{N} \nonumber \\
	&=&
		\frac{\frac{\sum_{i : \sigma_i = 0}r_i}{N_0}
		+ \frac{\sum_{i : \sigma_i = 0}r_i}{N_1} -
		\frac{\sum_{i=1}^N r_i}{N_1}
		}{N}
		=\frac{\condavg{r}{\sigma = 0}}{N_1}
		-\frac{N+1}{2N_1} \nonumber \\
	& \implies &
		\frac{\condavg{r}{\sigma = 0}}{N_1}=
		\frac{\Delta}{N} +\frac{N+1}{2N_1}.
		\label{AUC_Delta_eq1}
\end{eqnarray}
If we substitute Equation \eqref{AUC_Delta_eq1} into Equation \eqref{Delta:eq0}, we obtain
\begin{align}
	AUROC&=\frac{\condavg{r}{\sigma = 0}}{N_1}-\frac{N_1+1}{2N_0}
				= \Delta+\frac{N+1}{2N_1}-\frac{N_0+1}{2N_1} \nonumber \\
	&= \frac{\Delta}{N} + \frac{1}{2},
\end{align}
which completes the proof of the theorem.
\end{proof}

The equivalence of $\Delta$ to canonical statistical measures is not unique to the AUROC.
Indeed, we show that the Mann Whitney $U$ statistics can be computed from $\Delta$.

\begin{lemma}\label{lem:mwu}

	Let the Mann Whitney $U$ statistic computed from sample ranks of class 0 be designated as $U_0$.  Then $U_0$ is related to $\Delta$ by,
	\begin{align}
		U_0 = \frac{N_1N_0}{N}\left[\Delta + \frac{N}{2} \right]
	\end{align}
	where $N$ is the total number of samples, $N_1$ is the number of samples from class 1, $N_0$ is the number of samples from class 0, and $\Delta$ is defined according to Def~\ref{def:delta}.
\end{lemma}
\begin{proof}
	Recall the Mann Whitney $U$ statistic,
	\begin{align}
		U_0 &= \sum_{i:\sigma_i=0} r_i - \frac{N_0(N_0 + 1)}{2}\nonumber
	\end{align}
	where $r_i$ are the ranks of samples from class 0.  The sum of sample ranks is simply $N_0\condavg{r}{1}$.  By substitution,
	\begin{align}
		U_0 &= N_0\left[\condavg{r}{0} -
					\frac{N_1 + N_0 + 1}{2} -\frac{1}{2} +
					\frac{N_1 + 1}{2} \right],\nonumber\\
				&= N_0\left[
					\condavg{r}{0} - \avg{r} + \frac{N_1}{2}
				\right].\nonumber
	\end{align}
	Where, $\avg{r}$ may be written in terms of the conditional average ranks and the prevalence of class 1 ($\rho = N_1/N$) by, $\avg{r} = \rho\condavg{r}{1} + (1-\rho)\condavg{r}{0}$, and subsituting,
	\begin{align}
		U_0 &= N_0\left[
			\condavg{r}{0} - \rho\condavg{r}{1}-(1-\rho)\condavg{r}{0}+\frac{N\rho}{2}
		\right]\nonumber\\
		&= N_0\rho\left[ \left(\condavg{r}{0}-\condavg{r}{1}\right)
			+ \frac{N}{2}
		\right]\nonumber\\
		&= \frac{N_0N_1}{N}\left[\Delta + \frac{N}{2}\right]\nonumber
	\end{align}
	gives our desired result.
\end{proof}
While Lemma~\ref{lem:mwu} in itself is interesting, it also provides an additional proof of Theorem~\ref{thm:auc}.  This is due to the relationship of the $U$ statistic and the AUROC, \citep{hanley1982meaning,QJRMS2002Mason},
\begin{align}\label{eq:u_auc}
	AUROC &= \frac{U_0}{N_0 N_1}
\end{align}
which, by substituting Lemma~\ref{lem:mwu} into Equation~\eqref{eq:u_auc} results in Theorem~\ref{thm:auc}.

Theorem~\ref{thm:auc} provides the correspondence between $\Delta$ and AUROC which we will rely upon throughout the manuscript.  The proceeding theory, however, will use $\Delta$ as the performance metric due to mathematical convenience.  All of our results may be transformed to AUROC by applying Theorem~\ref{thm:auc}.

%
%

\subsection{Conditionally independent classifiers}

Each classifier in the ensemble $\set{f_i}_{i=1}^M$ assigns sample $k$ a rank in accordance to its respective class label, $\sigma_k$.  If any subset of classifiers, assigns ranks to sample k  $(1 \leq k \leq N)$ independently, then their conditional distribution factorizes,
\begin{align}
  \prob{f_i(x_k) = r, f_j(x_k) = s, \cdots, f_q(x_k) = l | \sigma_k} &=
    	P_i(r | \sigma_k)
    	P_j(s | \sigma_k)
			\cdots
			P_q(l | \sigma_k).\label{def:inde}
\end{align}
Consequently, we may compute the $n^{th}$ cross moment of such classifiers as a product of their respective conditional moments.
%
%
We will use this decomposition property of conditionally independent classifiers to show that $\Delta_i$ can be estimated from the central moments in the succeeding sections.
Next Theorem proves this decomposition property.

%
%
%
\begin{theorem}\label{thm:indep}
Suppose we are given $n$ ($3\leq n$) classifiers that are  $n$th order conditionally independent, i.e.
\begin{align}
  \prob{f_1(x_k) = s_1, \cdots,f_n(x_k) = s_n | \sigma_k} &=
   \prod_{i=1}^n P_i(s_i | \sigma_k)
   \label{def:inden},
\end{align}
and for $l \leq n$, let the $l^{th}$ order covariance central moment tensor, $Q_{l}$ ,defined as
\begin{align}
  Q_{l}(1,\cdots,l) &= \biggl \langle \bigl(r_{1}-\avg{r_{1}}\bigr)\cdots \bigl(r_{l}-\avg{r_{l}}\bigr)\biggr \rangle,\label{def:Qt}
\end{align}
where w.l.o.g. we denoted a given subset of methods of size $l$ by the set $\{1,\cdots,l\}$.
Then for any $2 \leq l \leq n$,
\begin{equation}
Q_{l}(1,\cdots,l)=\rho(1-\rho)(\rho^{l-1}-(\rho-	1)^{l-1}) \prod_{j=1}^l \Delta_{j}.
\end{equation}
\end{theorem}
%
%
\begin{proof}
We denote the $n$ methods as $\{1,...,n\}$ and proceed by induction on $l$. First note that the $n$th order conditional independence assumption implies that any given subset of methods $\{1,...,n\}$ are conditionally independent. Next, let $l=2$, and assume the two methods $1$ and $2$ are conditionally independent as such equation~\eqref{def:inden} is satisfied with $n=2$.  Then by using law of total expectation, we obtain the following:
\begin{eqnarray}  \label{eq:ci}
	Q_{2}(1,2) &=& \avg{(r_{1}-\avg{r_{1}})
	(r_{2}-\avg{r_{2}})}
	=\avg{r_1 r_2}-\avg{r_1}\avg{r_2} \nonumber \\
	&=&\avg{r_1 r_2|\sigma = 1}\rho+
	\avg{r_1 r_2|\sigma = 0}(1-\rho)
	-
	\avg{r_1}\avg{r_2}\nonumber \\
	&=&\avg{r_1|\sigma = 1}
	\avg{r_2|\sigma = 1}(\rho)+
	\avg{r_1 |\sigma = 0}
	\avg{r_2 |\sigma = 0}
	(1-\rho)
	-
	\avg{r_1}\avg{r_2} \label{pf:indep:eq1}
\end{eqnarray}
where we denote by $\rho$ the prevalence of class $\sigma$=1.  Similarly,  from the law of total expectation we have
\begin{eqnarray}
	\avg{r_1}&=&\avg{r_1|\sigma = 1} \rho +
	\avg{r_1| \sigma = 0}(1-\rho). \nonumber \\
\avg{r_2}&=&\avg{r_2|\sigma = 1} \rho +
	\avg{r_2| \sigma = 0}(1-\rho).
    \label{pf:indep:eq2}
\end{eqnarray}
Substituting Equation \eqref{pf:indep:eq2} into Equation \eqref{pf:indep:eq1}, we obtain
\[
	Q_{2}(1,2)=\rho(1-\rho)
	\Delta_1 \Delta_2.
\]
To ease the notation for the inductive step, for any $l \geq 1$, let $R_l$ be defined as follows
\[
R_l=(r_1-\avg{r_1}) \cdots (r_l-\avg{r_l})=(r_l-\avg{r_l}) R_{l-1},
\]
where $R_0=1$ and we used the fact that the methods $\{1,\cdots,l\}$ are conditionally independent.
Also note that $Q_l(1,\cdots,l)=\langle R_l \rangle$
Next, we inductively show two formulas which are required for the proof of the theorem.\\
\textbf{Claim 1:} For any $1 \leq l \leq n,$ $\avg{R_l| \sigma = 0} = \rho^l \prod_{i=1}^l \Delta_i.$\\
\textbf{Proof:} For $l=1$,
\begin{eqnarray}
	\avg{R_1| \sigma=0}&=&
	\avg{r_1-\avg{r_1}| \sigma=0}= \avg{r_1 | \sigma=0}-\avg{r_1} \nonumber \\
		&=& \avg{r_1 | \sigma=0}-
			\rho \avg{r_1 | \sigma=1}
			-(1-\rho) \avg{r_1 | \sigma=0} \nonumber \\
		&=&\rho( \avg{r_1 | \sigma=0}- \avg{r_1 | \sigma=1})=\rho \Delta_1. \label{pf:indep:eq3}
\end{eqnarray}
Next assume the claim is true for $l-1$, then from the inductive hypothesis we have
\begin{equation} \label{pf:indep:eq4}
	\avg{R_{l-1}| \sigma = 0} = \rho^{l-1} \prod_{j=1}^{l-1}\Delta_j.
\end{equation}
Let us now prove that if it is true for $l-1$, then it is true for $l$. From conditional independence and total law of expectation we get
\begin{eqnarray}
	\avg{R_l| \sigma=0}&=&
	\avg{(r_l-\avg{r_l})R_{l-1}| \sigma=0}
		= \left(\avg{r_l | \sigma=0}-\avg{r_l}\right)
			\avg{R_{l-1} | \sigma=0} \nonumber \\
	&=& \avg{r_l| \sigma=0}
	\avg{R_{l-1}| \sigma=0}-
	\avg{r_l}
	\avg{R_{l-1}| \sigma=0} \nonumber \\
	&=&\avg{r_l| \sigma=0}
	\avg{R_{l-1}| \sigma=0}-
	(1-\rho)\avg{r_l| \sigma=0}
	\avg{R_{l-1}| \sigma=0}-
	\rho\avg{r_l| \sigma=1}
	\avg{R_{l-1}| \sigma=0} \nonumber \\
	&=&\rho\avg{R_{l-1}| \sigma=0}(\avg{r_l| \sigma=0}-\avg{r_l| \sigma=1})\nonumber \\
	&=&\rho\avg{R_{l-1}| \sigma=0}\Delta_l=\rho \Delta_l\rho^{l-1} \prod_{j=1}^{l-1}\Delta_j=\rho^{l} \prod_{j=1}^{l}\Delta_j,\nonumber
\end{eqnarray}
where the previous to last equality follows from inductive assumption and completes the proof.\\

\noindent\textbf{Claim 2:} For any $1 \leq l \leq n,$ $\avg{R_n| \sigma = 1} = (\rho-1)^n \prod_{i=1}^l \Delta_i.$\\
\textbf{Proof:} For $l=1$,
\begin{eqnarray}
	\avg{R_1| \sigma=1}&=&
	\avg{r_1-\avg{r_1}| \sigma=1}= \avg{r_1 | \sigma=1}-\avg{r_1} \nonumber \\
	&=& \avg{r_1 | \sigma=1}-
	\rho \avg{r_1 | \sigma=1}
	-(1-\rho) \avg{r_1 | \sigma=0} \nonumber \\
	&=&-(1-\rho)( \avg{r_1 | \sigma=0}- \avg{r_1 | \sigma=1})=(\rho-1) \Delta_1. \label{pf:indep:eq7}
\end{eqnarray}
Next assume the claim is true for $l-1$, then from the inductive hypothesis we have
\begin{equation}
	\avg{R_{l-1}| \sigma = 1} = (\rho-1)^{l-1} \prod_{i=1}^{l-1}\Delta_i.
\end{equation}
For $l$, by the conditional independence and total law of expectation we get
\begin{eqnarray}
	\avg{R_l| \sigma=1}&=&
	\avg{(r_l-\avg{r_l})R_{l-1}| \sigma=1}= (\avg{r_l | \sigma=1}-\avg{r_l})\avg{R_{l-1}| \sigma=1} \nonumber \\
	&=&\avg{r_l| \sigma=1}
	\avg{R_{l-1}| \sigma=1}-
	\avg{r_l}
	\avg{R_{l-1}| \sigma=1} \nonumber \\
	&=&\avg{r_l| \sigma=1}
	\avg{R_{l-1}| \sigma=1}-
	\rho\avg{r_l| \sigma=1}
	\avg{R_{l-1}| \sigma=1}-
	(1-\rho)\avg{r_l| \sigma=0}
	\avg{R_{l-1}| \sigma=1} \nonumber \\
	&=&(\rho-1)\avg{R_{l-1}| \sigma=1}(\avg{r_l| \sigma=0}-\avg{r_l| \sigma=1})\nonumber \\
	&=&(\rho-1)\avg{R_{l-1}| \sigma=1}\Delta_l=(\rho-1) \Delta_l(\rho-1)^{l-1} \prod_{j=1}^{l-1}\Delta_j=(\rho-1)^{l} \prod_{j=1}^{l}\Delta_j,\nonumber
\end{eqnarray}
Now we are ready to prove the main theorem.
Using law of total expectation and claim $1$ and $2$, we obtain the following set of equations
\begin{eqnarray}
	\avg{R_l}&=&\avg{R_l | \sigma=1}\rho +\avg{R_l | \sigma=0}(1-\rho)=(\rho(\rho-1)^l+(1-\rho)\rho^l)\prod_{j=1}^l\Delta_j\nonumber \\
	&=&\rho(1-\rho)(\rho^{l-1}-(\rho-1)^{l-1})\prod_{j=1}^l\Delta_j,
\end{eqnarray}
which completes the proof of the theorem.
\end{proof}
%

Theorem~\ref{thm:indep} shows that under the assumption of conditionally independent ensemble members the $n^{th}$ central moment contains information of each methods performance, $\Delta_i$.  In addition, note the symmetry $\rho \rightarrow 1-\rho$, and $\Delta \rightarrow -\Delta$ of the l-th centered moment. This symmetry is expected because the classes 1 and 0 could be called 0 and 1, and the same results should hold.  The next section leverages on various strategies to retrieve estimates of $\Delta_i$ from data.
%
\subsection{Ranking Classifiers using the Covariance Matrix}
In the previous section, we showed that $\Delta$ is both an intuitive and statistically principled measure of classifier performance.  However, direct calculation of $\Delta$, from its definition, requires knowledge of each sample's class label.  In the current section we show how to estimate a value $v_i$ that is proportional to $\Delta_i$ using the covariance matrix of predicted sample ranks.  For this we rely upon Theorem~\ref{thm:indep}, which states a relationship between $\Delta_i$ and the central moments of rank predictions from conditionally independent base classifiers.

Recall that $Q_2$ represents the covariance matrix of the ranked predictions from each base classifier, i.e.
\begin{align}\label{def:Q}
  Q_2(i, j) &= \avg{(r_{i}-\avg{r_{i}})
	(r_{j}-\avg{r_{j}})},
\end{align}
for given two methods $i$ and $j$.
For an arbitrary set of base classifiers, not necessarily conditionally independent, their covariance consists of the both intraclass and interclass correlation of rank predictions.  The intraclass correlation is often a manifestation of similarities between the algorithms or training data of base classifier pairs.  Here the methods tend to rank samples given either latent class together.  Whereas the interclass correlation represents the agreement of the average sample rank predictions for each latent class.  In other words, pairs of base classifiers assign a sample rank less than or greater than $N/2$ for samples from latent class 1 and 0, respectively.  Therefore, if the base classifiers are conditionally independent the intraclass correlation vanishes, making the measured correlation exclusively attributable to their mutual agreement of ranks based upon each sample's latent class label. This was the essence of Theorem~\ref{thm:indep}, and will be utilized in this section.

%
%
\begin{theorem} \label{thm:r1}
	The covariance matrix $Q_2$ as defined in Equation \eqref{def:Q} can be expressed as,
  \begin{equation}\nonumber
    Q_2(i, j) =
    \begin{cases}
      \frac{N^2 - 1}{12} & \text{if } i = j\\
      \rho(1-\rho) \Delta_i\Delta_j &
        \text{if } i\neq j
    \end{cases}
  \end{equation}
  if all the methods are mutually conditionally independent.
\end{theorem}
\begin{proof}
	Suppose $i \neq j$, then the result follows from Theorem~\ref{thm:indep} by letting $l=2$.  Next, let $i = j$ and recall that $r_i$ is a vector in which each element represents the unique rank of each sample $k \in \set{1, \dots, N}$.  Then, the variance, $Q_2(i, i)$, of $r_i$ is that of an uniform discrete distribution, $(N^2-1)/12$, which completes the proof of the theorem.
\end{proof}
Theorem~\ref{thm:r1} states that the agreement of sample ranks predicted by pairs of base classifiers is due to their independent ability to rank samples by their latent class label (similar results were found for binary predictions in \citep{parisi2014ranking}).    This can be made intuitive by a simple example and counter example. Consider two base classifiers, one that perfectly assigns a rank to samples according to the sample's class and the other that assigns sample ranks at random.  Despite one of the methods perfect classification, the resulting correlation is zero due to the random assignments of the uninformative base classifier.  This intuition is captured by the product of their performances, which in this example is zero.  However, if instead both classifiers perfectly rank samples according to their latent class label, their predictions will be maximally correlated.

Inspection of Theorem~\ref{thm:r1} motivates a strategy for estimating each base predictors performance metric, $\Delta$, from the covariance matrix.  Here we see that $Q_2$ can be decomposed to
\begin{align}
	Q_2 = R - diag(R) + \frac{N^2 - 1}{12}I
\end{align}
where $I$ is the identity matrix and $R$ is a rank one matrix $\lambda vv^T$, with
\begin{align}
  \lambda &:= \rho(1-\rho)\sum_{j = 1}^M \Delta_j^2\label{eq:lambda}
\end{align}
and
\begin{align}
  v_i :=\frac{\Delta_i}{\sqrt{\sum_{j = 1}^M \Delta_j^2}} .\label{eq:eigvector}
\end{align}
Consequently, inferring the performance of each base classifier amounts to estimating the eigenvector $v$ corresponding to $R$ whose off-diagonal entries are identical to the covariance matrix. In the next section we provide two strategies for estimating the diagonal entries of $R$ from real data and derive a maximum likelihood estimator using this information.

Theorem~\ref{thm:r1} represents a sufficient, when $M>=4$, but not a necessary condition for estimating each base classifier's performance. 
In other words, if the classifiers are conditionally independent, then we take advantage of the decomposition in Theorem~\ref{thm:r1} to estimate the performance of the base classifiers.
Note that the condition that $M>=4$ is important because if $M<4$ the system is under-determined.

Before we present our next result, we will introduce some notation.  Throughout the rest of the paper we say that a matrix $A$ has a rank-one approximation if there exists a vector $q$ and a diagonal matrix $D$ s.t.
\begin{equation}
	A=q q^t + D.\nonumber
\end{equation}

%
%
\begin{theorem} \label{thm:r1suf}
  Suppose that the covariance matrix $Q_2$ has a rank-1 approximation, i.e. there exists a vector $q$ and a diagonal matrix $D$ such that
  \begin{equation}
    Q_2 = qq^T + D.\nonumber
  \end{equation}
If one base classifier (classifier 1) is independent of all the others (classifiers i $\neq$ 1), and two additional classifiers (classifiers 2 and 3) are conditionally independent, then the performance of the $i^{th}$ method is given by $q_i=\sqrt{\rho(1-\rho)}\Delta_i$.
\end{theorem}
\begin{proof}
Let $Q_2=qq^T+D$, and consider the entries of the first row of $Q_2(1, i) = q_1q_i$ for $i \neq  1$, where
  \begin{equation*}
    q_1q_i=\rho(1-\rho)
      \Delta_1 \Delta_i \quad \implies
      q_i=\frac{\rho(1-\rho)\Delta_1\Delta_i}{q_1},
      \quad
  \end{equation*}
  and
  \begin{equation} \label{eq:r1nec}
   q_1q_2=\rho(1-\rho)
      \Delta_1 \Delta_2, \quad
       q_1q_3=\rho(1-\rho)
      \Delta_1 \Delta_3,
   \end{equation}
  where we used the arguments of Theorem~\ref{thm:r1} for row $1$ and column $i$.  
 Similarly, since classifiers $2$ and $3$ are independent, we have
  \begin{equation}
  	q_2q_3 = \rho(1-\rho) \Delta_2 \Delta_3,\nonumber
  \end{equation}
and  combining this with Equation \eqref{eq:r1nec}, we obtain
  \[
    \frac{q_1q_2}{q_2q_3} = \frac{\Delta_1}{\Delta_3}
    \implies     \frac{q_1}{q_3} = \frac{\Delta_1}{\Delta_3}\quad \text{and} \ \quad
 q_1q_3=\rho(1-\rho)
      \Delta_1 \Delta_3   ,
  \]
  from which we obtain $q_1=\sqrt{\rho(1-\rho)} \Delta_1$ and by substituting into Equation \eqref{eq:r1nec} we find that $q_i=\sqrt{\rho(1-\rho)}\Delta_i$. In like, ensembles of greater than three base classifiers can be identically solved, which completes the proof.
\end{proof}

In this section, we have elucidated the necessary and sufficient conditions for estimating the performance of each classifier without knowledge of each sample's class label (Theorems~\ref{thm:r1} and~\ref{thm:r1suf}).  Moreover, we have shown how we can use the covariance matrix in order to estimate a vector whose entries are proportional to the performance of the individual methods.  We proceed by developing an optimal aggregation method that incorporates the estimated performances for constructing an ensemble classifier.


\subsection{Strategy for Unsupervised Multiple Method Aggregation: SUMMA}
In this section we develop an optimal meta-learner to infer each sample's latent class.  As in the seminal work of \citep{AppStats1979Dawid} and subsequent work by others \citep{IER1982Nitzan,parisi2014ranking} the inference method amounts to maximizing the likelihood of the class label given a set of observations.  These being the rank predictions of each conditionally independent base classifier, which amounts to
\begin{align}
  \sigma_k^{\text{MLE}} &= \argmax{\sigma_k}{
      \sum_{i=1}^M \log \left( P_i\left(r_k | \sigma_k \right)\right)
    }. \nonumber
\end{align}
And, by application of Bayes' Theorem the MLE can be equivalently written as,
\begin{align}\label{eq:max_like}
  \sigma_k^{\text{MLE}} &= \Theta \left\{
      M\log\left(\frac{N_0}{N_1}\right) +
      \sum_{i=1}^M \log \left(
        \frac{P_i\left(\sigma_k=1 | r_k\right)}{1-P_i\left(\sigma_k=1 | r_k\right)}
      \right)
    \right\}.
\end{align}
Here, each base classifier's performance is incorporated in the respective conditional distribution.  Intuitively, the log ratio of accurate methods contributes more to the sum than do inaccurate methods.  For example, the log ratio of an uninformative method, $P_i(\sigma = 1 | r) = 0.5$, will not contribute any mass to the sum, while the contribution of methods diverge as $P_i(\sigma=1|r)$ goes to 1.

A challenge to implementing this strategy is that the conditional probability distributions are {\it a priori} unknown.  Rather than assuming that the rank predictions are distributed according to an arbitrary distribution, we {\it infer} a distribution that requires the minimum set of {\it a priori} assumptions while reproducing the ``known" statistical quantities.  The maximum entropy methodology provides the necessary tools for such a task \citep{PhysRevLett1957Jaynes}.
%
%
\begin{lemma}\label{lem:prob}
  The maximum entropy probability distribution of the latent class label given rank for the $i^{th}$ ensemble member, is
  \begin{align}\nonumber
    P_i(\sigma=1 | r) &= \frac{1}{
        1+\ex{
          \frac{3\Delta_i}{\avg{r}^2}
            \left(
            r - \avg{r}
            \right) +
          \log\left(\frac{N-N_1}{N_1}\right)
        }
      }.
  \end{align}
  where $N_1$ of $N$ samples belong to class 1.
\end{lemma}
\begin{proof}
Consider a method that rank orders $N$ samples, as $N$ gets very large the difference between the statistical moments of the discrete rank distribution and its continuous analog become smaller. We begin by taking $N$ sufficiently large that the maximum entropy distribution is approximately inferred by the maximum entropy functional,
  \begin{align}\nonumber
    J &= -\int_{r=0}^N \sum_{\sigma\in\set{0, 1}} P(r) P_i(\sigma| r)\log(P_i(\sigma| r))dr + \sum_{i=0}^2 \lambda_i\Lambda_i
  \end{align}
in which we aim to infer $P_i(\sigma | r)$ subject to the constraints $\Lambda_i$ for $i=(0, 1, 2)$, where
  \begin{subequations}
    \begin{align}
      \Lambda_0 &= N - \int_{r=0}^N\sum_{\sigma\in\set{0, 1}}
        P_i(\sigma|r), \label{con:norm}\\
      \Lambda_1 &= N_1 - \int_{r=0}^N\sum_{\sigma\in\set{0, 1}}
        \sigma P_i(\sigma|r), \label{con:mass_conserv}\\
      \Lambda_2 &= N_1 \condavg{r}{\sigma=1}_i - \int_{r=0}^N r\sum_{\sigma\in\set{0, 1}}
        \sigma P_i(\sigma| r), \qquad \text{or equivalently }\nonumber\\
        &= N_1 \left(\condavg{r}{\sigma=0}_i - \Delta_i\right) - \int_{r=0}^N r\sum_{\sigma\in\set{0, 1}}
          \sigma P_i(\sigma| r).\label{con:avg}
    \end{align}
  \end{subequations}
  Here we constrain the inferred function such that the sum over labels normalizes to 1 (Equation~\ref{con:norm}), the average occurrence of class one samples is reflective of its occurrence in the population ($N_1$) (Equation~\ref{con:mass_conserv}), and lastly that the average rank condition by the class label recovers our empirical estimate of $\Delta_i$ (Equation~\ref{con:avg}).

  We consider the criterion satisfied when the variation of the Lagrangian functional, $\delta J$, with respect to $P_i(\sigma|r)$ is stationary,
  \begin{align}\nonumber
    \delta J &= \int_{r=0}^N\sum_{\sigma=0}^1 \left[
      -\frac{1}{N}\log\left(P_i(\sigma|r)\right) - \frac{1}{N} -
      \lambda_0 -
      \lambda_1 \sigma -
      \lambda_2 \sigma r
      \right] \delta P_i(\sigma|r) = 0
  \end{align}
  and consequently,
  \begin{align}\label{eq:maxent}
    P_i(\sigma | r) &= \frac{1}{Z} e^{-\lambda_1\sigma - \lambda_2\sigma r}.
  \end{align}
  In Equation~\eqref{eq:maxent}, $\lambda_0$ is incorporated in the constant $Z$, which satisfies the normalization constraint in Equation~\eqref{con:norm}.  Therefore, $Z = 1+e^{-\lambda_1 - \lambda_2 r}$.

  Next, we compute the remaining Lagrange multipliers by substituting in Equation~\eqref{eq:maxent} into Equations~(\ref{con:mass_conserv}, \ref{con:avg}).  Each of these equations, after summing over class labels, amount to calculating integrals over $P_i(\sigma=1|r) = (1 + e^{\lambda_1 + \lambda_2 r})^{-1}$, which for transparency we will refer to as $f_i(r)$.

	To achieve interpretable and analytic solutions we approximate these integrals by expanding $f_i$ to first order about $\avg{r}$, resulting in $f_i(r) \approx f_i(\avg{r}) + f_i'(r)|_{r =\avg{r}} \delta r$, with $f_i'(r) = -\lambda_2 f_i(r)(1-f_i(r))$.  Indeed, the first order expansion captures our intuition from the maximum entropy distribution in Equation~\ref{eq:maxent}.  To see this, consider an uninformative base classifier, that is it ranks samples without regard to their latent class label.  Such a classifier would simply model the data by the uniform distribution over $N$, i.e. $\lambda_2 = 0$. Instead, now consider a base classifier $f_i(r)$ that is weakly predictive, $0 < \lambda_2 \ll 1$, intuitively the form of such a base classifier is captured by a shallow and negatively sloped line about the mean.  Indeed, in applying this intuition we find that the first order approximation to Equation~\ref{con:mass_conserv} is,
\begin{align}
\Lambda_1 &\approx N_1 - N f_i(\avg{r})\left[
1 + \lambda_2\left(1 - f_i(\avg{r})\right) \avg{\delta r}
\right] \nonumber
\end{align}
where $\avg{\delta r} = 0$.  The expansion then results in two useful forms of Equation~\ref{con:mass_conserv},
\begin{subequations}
\begin{align}
\frac{N_1}{N} &= f_i(\avg{r})\label{con:simp_f}\\
\log\left(\frac{N - N_1}{N_1}\right) &= \lambda_1+\lambda_2 \avg{r}.\label{con:mass_conserv_simp}
\end{align}
\end{subequations}
In like fashion, we find that applying our first order approximation to Equation~\ref{con:avg} amounts to,
\begin{align}
\Lambda_2 &= N_1 \big(\condavg{r}{\sigma=0}_i - \Delta_i\big) - 
N f_i(\avg{r})\big[
\avg{r} - \lambda_2(1 - f_i(\avg{r}))\avg{\delta r^2}
\big] \label{con:simp_avg}
\end{align}
where, the second central moment, $\avg{\delta r^2}$, is simply that of a uniform distribution, $\avg{\delta r^2} = (N^2)/12$, or equivalently $\avg{\delta r^2}= \avg{r}^2/3$.  Then, by substituting the second central moment and Equation~\ref{con:simp_f} into Equation~\ref{con:simp_avg} we find that,
\begin{align}
\Lambda_2 &= N_1\left(
\condavg{r}{\sigma=0}_i - \Delta_i - \avg{r} + \lambda_2\left(\frac{N-N_1}{N}\right) \frac{\avg{r}^2}{3}
\right).
\end{align}
Recognizing that $\Lambda_2 = 0$, we see that an algebraic solution for $\lambda_2$ is feasible.  If we rearrange the terms,
\begin{align}
\Delta_i + \avg{r} - \condavg{r}{\sigma=0}_i &= \lambda_2 \left(\frac{N-N_1}{N}\right)\frac{\avg{r}^2}{3}
\end{align}
and make the substitution $\avg{r} = \rho\condavg{r}{1}_i + (1-\rho)\condavg{r}{0}_i$ where $\rho = N_1/N$,
\begin{align}
\Delta_i(1-\rho) &= \lambda_2(1-\rho)\frac{\avg{r}^2}{3}.
\end{align}
Solving the above equation for $\lambda_2$, and applying the solution to Eq~\ref{con:mass_conserv_simp} we find that,
  \begin{align}
    \lambda_1 &= -\frac{3\Delta_i}{\avg{r}} + \log\left(\frac{N-N_1}{N_1}\right),\nonumber\\
    \lambda_2 &= \frac{3\Delta_i}{\avg{r}^2}.\nonumber
  \end{align}
While the linear approximation was used for the estimation of the Lagrange multipliers, the distribution $P(\sigma |r)$ is still the logistic function of Equation~\ref{eq:maxent}.  The proof is completed by substituting $\lambda_1$ and $\lambda_2$ into $P_i(\sigma=1|r)$.
\end{proof}
Next we apply the maximum entropy probability distribution to derive a maximum likelihood estimator of each sample's latent class label.
%
%
\begin{theorem}\label{theorem:mle}
  With all the definitions of previous sections the maximum likelihood estimator (MLE) of the $k^{th}$ sample's latent class label $\sigma_k$ is given as,
  \begin{align}
    \hat{\sigma}_k^{SUMMA} &= \Theta\left\{
      \sum_{i=1}^M v_i \left(
        \avg{r} - r_{ik}
      \right)
    \right\},\nonumber
  \end{align}
  with $\Theta$ representing the Heaviside step function.  We denote the estimated class label SUMMA, as it is estimated by the Strategy for Unsupervised Multiple Method Aggregator.
\end{theorem}

\begin{proof}
  Applying Lemma~\ref{lem:prob} to the our maximum Likelihood estimator in Equation~\eqref{eq:max_like},
  \begin{align}
    \sigma_k^{\text{MLE}} &= \Theta \left\{
    M\log\left(\frac{N_0}{N_1}\right) -
    \sum_{i=1}^M
      \frac{3\Delta_i}{\avg{r}^2} \left(
        r_{ik} - \avg{r}
      \right)
       + \log\left(\frac{N-N_1}{N_1}\right)
      \right\}\nonumber
  \end{align}
	or equivalently,
	\begin{align}
		\sigma_k^{\text{MLE}} &= \Theta\left\{
				\frac{3}{\avg{r}^2} \sum_{i=1}^M \Delta_i\left(\avg{r} - r_{ik}\right)
			\right\}\nonumber
	\end{align}
  and recall from Equations~(\ref{eq:lambda}, \ref{eq:eigvector}) that $\Delta_i = \sqrt{\frac{\rho(1-\rho)}{\lambda}} v_i$.  Then by substitution,
  \begin{align}
    \sigma_k^{\text{MLE}} &= \Theta \left\{
    \frac{3}{\avg{r}^2}\sqrt{\frac{\lambda}{\rho(1-\rho)}}\sum_{i=1}^M
      v_i \left(
        \avg{r} - r_{ik}
      \right)
    \right\} \label{eq:MLE1}
  \end{align}
	where the terms preceeding the sum have no influence to the image of the argument under the Heaviside step function, and consequently may be ignored.  This result completes the proof.
\end{proof}


\section{Methods for Estimating Performances from the Covariance Matrix}
The center piece of the SUMMA algorithm is that the off-diagonal elements of covariance matrix of ranked predictions are those of a rank one matrix ($R$).  Moreover, the eigenvector elements of $R$ are proportional to our performance metric $\Delta$.  Despite the elegance of our solution, we have yet to demonstrate how to estimate $R$.  Indeed, spectral decomposition of the covariance matrix does not produce the vector we are seeking, because $Q_2(i,i) \neq R(i,i)$ $\forall i$.  Consequently, we are required to {\it infer} the diagonal of $R$, which then yields the vector of interest.  In the following sub-sections we present two methods to infer the diagonal entries of $R$.  Our first approach formulates this task as a semi-definite program (SDP) which then can be solved using any SDP solver.
Our second strategy is coming from the recent literature on the famous matrix completion problem, see e.g. \citep{candes2009exact,ha2017alternating,cai2010singular,jain2010guaranteed}. It is an iterative method that only requires us to calculate the largest singular value and associated singular vector at each iteration. Together, these strategies provide the reader with i) methods for estimating the values proportional to the performance metric $\Delta$ and ii) guarantees that the numerical solutions are relevant.


\subsubsection{Semi-Definite Programming}\label{app:sdp}
In this section, we estimate the rank-1 matrix $R = \lambda vv^T$ from the covariance matrix $Q_2$ using convex optimization. To start, recall that the covariance matrix is of the form $Q_2 = R + D$, where $D$ is a diagonal matrix.
We would like to estimate $D$ from the covariance matrix $Q_2$ such that $Q_2-D$ is equal to $R$.  Mathematically, we may formalize the problem as follows,
\begin{eqnarray} \label{optim:greedy}
  \min_{D} rank(Q_2 - D) \quad \text{s.t.} \quad
			D(i,j) = 0 \quad \text{for} \quad i \neq j, \quad
			\text{and} \quad Q_2 - D\succeq 0,
\end{eqnarray}
where  $Q_2 - D \succeq 0$ means that the matrix $Q_2 - D$ is positive semidefinite. The positive semi-definite constraint is required from our theoretical results - the eigenvalue corresponding to $R$ is positive, see equation \eqref{eq:lambda}. Next, we will show that under the assumption of positive semi-definiteness of $R$ the corresponding diagonal matrix $D$ is unique.

%

%
%
\begin{theorem} \label{thm:rankmin}
	Let $Q$ be a $M \times M$ approximately rank-1 symmetric matrix,
	\begin{equation} \label{eq1}
	  Q= qq^T + D_0,
	\end{equation}
	for some $q \in \R^m$ such that $m \geq 3$ and diagonal matrix $D_0$.  Then the optimization problem in Equation \eqref{optim:greedy} has the unique solution $D_0$, with $Q-D_0= qq^T =(-q)(-q)^T$ so that we can recover $q$ up to its sign.
\end{theorem}
\begin{proof}
	Let $Q$ be defined as in Equation~\eqref{eq1} with $q \neq 0$ and $D_0$.  With $q\neq0$, it is obvious that for any diagonal matrix $D$, $Q - D \neq 0$.  Hence, $rank(Q - D) > 0$ for any diagonal matrix $D$.  Moreover, since $Q - D_0 = qq^T$ one possible solution of the optimization problem in Equation~\eqref{optim:greedy} is $D_0$.  Next, we show that $D_0$ is the only diagonal matrix that satisfies the feasibility conditions of the optimization problem.

Suppose there exists a diagonal matrix $D_1$ such that $D_1 \neq D_0$ and $rank(Q - D_1) = 1$.  Then since $Q - D_1$ is symmetric, there exists $\hat{q}$ such that $Q - D_1 = \hat{q} \hat{q}^T$ and $q \neq \hat{q}$.  Since both $D_0$ and $D_1$ are diagonal and $D_1 + \hat{q}\hat{q}^T = Q = D_0 + qq^T$, then the equality
	\begin{equation} \label{eq:offdiag}
  	q_iq_j = \hat{q}_i \hat{q}_j
	\end{equation}
	must be true for all $i \neq j$ where $i,j \in \set{1, 2, \dots, M}$.

	Without loss of generality lets assume that $q_1 > \hat{q}_1$, and $M \geq 3$.  Equation~\eqref{eq:offdiag} for $(i, j) = (1, 2)$ implies that $\hat{q}_2 > q_2$, and in like, for $(i, j) = (1, 3)$ implies that $\hat{q}_3 > q_3$.  Under these inequalities, when $(i, j) = (2, 3)$ we see that $\hat{q}_2 \hat{q}_3 > q_2 q_3$, which contradicts equation~\eqref{eq:offdiag}.  Therefore, $q=\hat{q}$ making $D_0 = D_1$ and consequently $D_0$ is the unique solution to the optimization problem, equation~\eqref{optim:greedy}.
\end{proof}

Note that Theorem~\ref{thm:rankmin} shows that $q$ can be recovered uniquely up to its sign. Without any further assumptions, it is impossible to determine the sign of each coordinate of $q$ since $qq^T=(-q)(-q)^T$. In this manuscript, we assume that the majority of base classifiers have adopted the right sign convention.  Consequently, we solve the ambiguity between $q$ and $-q$ based upon which have the greatest number of positive entries.

While the optimization problem described in \eqref{optim:greedy} is intuitive, in practice it is difficult to solve. This is because the rank function is not convex \citep{candes2009exact}, making the optimization NP-hard \citep{jain2010guaranteed}. To circumvent this shortcoming, we chose to optimize the convex relaxation of the rank function, namely the nuclear norm, as is done in the matrix completion literature \citep{candes2010power}. For a given matrix $A$ the nuclear norm is defined as,
\begin{equation} \label{def:nuclear}
	\norm{A}_*=\sum_{i=1}^M \sigma_i(A),
\end{equation}
where $\sigma_i(A)$ represents the $i^{th}$ singular value of $A$.  Accordingly, our original optimization problem can be relaxed to the following semidefinite program,
\begin{eqnarray} \label{optim2}
  \min_{D} \quad \norm{Q-D}_* \quad
			\text{s.t.} \quad D \text{ is diagonal, and } \quad
			Q-D \succeq 0.
\end{eqnarray}
Semidefinite programming has been a popular approach recently with the availability of efficient solvers and is used in a variety of disciplines including signal processing and control theory, \citep{vandenberghe1999applications}.  Next we characterize the solutions of the semidefinite program in Equation~\eqref{optim2}.
%
%
\begin{theorem} \label{thm:2}
	Suppose that $Q$ is an $M \times M$ matrix of the form Equation~\eqref{eq1} for some $q$ and diagonal matrix $D_0$, then	the optimization problem in Equation \eqref{optim2} has the unique solution $D_0$, provided that
	\begin{equation} \label{condition}
	q_i^2<\sum_{j\neq i}q_j^2, \quad \forall i \in \set{1, 2, \dots, M}
	\end{equation}
	Moreover, $Q - D_0=qq^T$ so that we can recover $q$ up to its sign.
\end{theorem}
\begin{proof}
Let $D$ be an arbitrary diagonal matrix such that $Q - D$ is a PSD (Positive Semidefinite) matrix.  Then the eigenvalues of $Q - D$ are non-negative and equal to the singular values of $Q - D$.  Combined with the fact that the trace of a matrix is equal to the sum of its eigenvalues, we know the following:
\[
\norm{Q - D}_* = Tr(Q- D).
\]
Without loss of generality lets assume that the diagonal entries of $Q$ are equal to 0, that is $Q = qq^T - diag(qq^T)$.  In this case, note that $D_0 = -diag(qq^T)$ is a feasible solution for the optimization problem in Equation \eqref{optim2}.  Next let $D$ be an arbitrary solution and since $Q-D$ is PSD, $D_{ii}\leq 0$ for all $i$.  Suppose for some $i$ we have $D_{ii} > -q_i^2$ and for all $j \neq i$ we have	$D_{jj} \geq -q_j^2$. For $j \neq i$, consider the following sub-matrix of $Q-D$,
	\[
		(Q-D)^{ij}=\begin{bmatrix}
		-D_{ii} & q_i q_j \\
		q_iq_j & -D_{jj}
		\end{bmatrix}.
	\]
Since $det((Q-D)^{ij})=D_{ii}D_{jj}-q_i^2q_j^2<0$, the submatrix $(Q-D)^{ij}$ of $Q-D$ is negative definite	which contradicts the fact that $Q-D$ is PSD.  Therefore, combined with the fact that $D_{ii} \leq 0$ for each $i$, this implies that in order a diagonal matrix $D$	to be a feasible point for the optimization problem \eqref{optim2}
	\begin{enumerate}
		\item Either for all $i$, we have $D_{ii} \leq -q_i^2$
		\item Or there exist $i$ such that  $D_{ii} > - q_i^2$, and
		$\forall j \neq i$ $D_{jj} \leq - q_j^2$.
	\end{enumerate}
Suppose $D$ is a feasible point that satisfies condition $1$, then
$Tr(Q-D)\geq Tr(Q+diag(qq^t))$ for every feasible	$D$.  Hence, $D_0=-diag(qq^t)$ remains to be the unique solution of the optimization problem.  Now suppose there exists a feasible point $D$ such that condition 2 of above is satisfied.	 Then WLO assume that $D_{11}>-q_1^2$ and $D_{jj}\leq -q_i^2$ for	$j \geq 2$. 	Next, for each $j \geq 2$ consider the following submatrix:
	\[
		(Q-D)^{1j}=\begin{bmatrix}
		-D_{11} & q_1 q_j \\
		q_1 q_j & -D_{jj}
		\end{bmatrix}.
	\]
	In order for $Q-D$ to be PSD, we should have $det((Q-D)^{1j})=D_{11}D_{jj}-q_1^2q_j^2>0$, which in turn implies that
	\begin{equation} \label{cond:1}
		D_{jj}<\frac{q_1^2q_j^2}{D_{11}}.
	\end{equation}
	Using Equation \eqref{cond:1}, we observe that
	\begin{eqnarray}
		\norm{Q-D}_*&=&Tr(Q-D)=-D_{11}-\sum_{j > 1}D_{jj} \nonumber\\
		&>&-D_{11}+\sum_{j > 1} -\frac{q_1^2q_j^2}{D_{11}} \label{eqc:1},
	\end{eqnarray}
	where Equation \eqref{eqc:1} comes from Equation \eqref{cond:1}.	 Now from the assumption that $q_1^2 <\sum_{j > 1} q_j^2$, we have $D_{11}>-q_1^2>-q_1 \sqrt{\sum_{j > 1} q_j^2}$, and for any $D_{11} \in [-q_1^2,0]$
	\begin{eqnarray}
		\norm{Q-D}_*&<&-D_{11}-\sum_{j > 1} \frac{q_1^2q_j^2}{D_{11}}
		> -q_1^2 -\sum_{j > 1} q_j^2=\norm{Q+diag(qq^t)}_* \label{eqcd:1}.
	\end{eqnarray}
	%
	%
	%
A generalization of the above argument implies that if for each $i$, $q_i^2<\sqrt{\sum_{j \neq i} q_j^2}$, then $\norm{Q-D}_*>\norm{Q+diag(qq^t)}_*$ so that $D_0=-diag(qq^t)$ the unique solution to the optimization problem \eqref{optim2}.
\end{proof}
Note that the above theorem does not hold if
\begin{equation} \label{eq:perf}
	\sum_{i \geq 2}  q_i^2 < q_1^2 .
\end{equation}
However, unless there is a near perfect classifier and some random classifiers it is very unlikely that Equation \eqref{eq:perf} is violated.

\subsubsection{An Iterative Approach}\label{app:iter}
The iterative algorithm we present in this section and its convergence results are from the matrix completion literature and can be found in the recent papers \citep{ha2017alternating,jain2010guaranteed}. Basically, as shown in \citep{ha2017alternating}, the iterative method we present converges with high probability except for some peculiar cases where only one classifier is better than random whereas the rest are random, i.e. $v_1>0,v_i=0 ~ \forall i>1$. Before presenting the iterative method, let us define the set $\Omega=\{(i,j)~:~i\neq j \in \{1,\cdots,m\} \}$.
Then, using the notation of \citep{ha2017alternating}, we can formulate the following matrix completion problem:
\begin{equation}
\min_{X} \frac{||
P_\Omega(X)-P_\Omega(Q_2) ||_F^2}{2} \quad s.t. \quad X \in C=\{X\in R^{M\times M} :rank(X)=1, X=X^T, X\succeq 0 \},\label{matrix_completion}
\end{equation}
where the linear operator $P_\Omega(X)$ is defined as
\begin{equation}
P_\Omega(X)=
\begin{cases}
      X_{ij} & if (i,j) \in \Omega \\
      0 & (i,j). \notin \Omega
   \end{cases}
\end{equation}
It is obvious that $R:=qq^t \in C$ and $||
P_\Omega(R)-P_\Omega(Q_2) ||_F^2=0$, as such $R$ is a minimizer of the optimization problem in \eqref{matrix_completion}.
Moreover, under the assumptions of Theorem~\ref{thm:rankmin}, it is clear that $R$ is the unique minimizer.
Therefore, by solving the optimization problem in \eqref{matrix_completion}, we can recover the unknown vector $q$ exactly.
The optimization problem in \eqref{matrix_completion} is exactly the matrix completion problem presented in
\citep{ha2017alternating}, where we only utilized the additional fact that the matrix $R$ we are observing has rank one. From Theorem~3 of \citep{ha2017alternating}, the following gradient descent algorithm converges to $R$:
\begin{eqnarray}
Y_t&=&X_t-\nabla \left(\frac{1}{2}||
P_\Omega(X_t)-P_\Omega(Q_2) ||_F^2\right) \nonumber \\
X_{t+1} &=& P_C(Y_t), \label{iterative_v1},
 \end{eqnarray}
where $P_C(Y_t)$ is the projection of $Y_t$ into the set $C$. First note that
\[
\nabla \left(\frac{1}{2}||P_\Omega(X_t)-P_\Omega(Q_2) ||_F^2 \right)=P_\Omega(X_t)-P_\Omega(Q_2),
\]
and from Eckart - Young - Mirsky Theorem \citep{eckart1936approximation}, we have
\[
P_C(Y_t)=\sigma_1(Y_t)u_1u_1^T,
\]
where $\sigma_1(Y_t)$ is the largest singular value of $Y_t$ and $u_1$ is the corresponding singular vector.
Using these two observations, the gradient descent algorithm given in \eqref{iterative_v1} is equivalent to the following.
\begin{algorithm}[H]
  \caption{Find rank 1 matrix from off-diagonal observations of covariance matrix}\label{alg:iter}
  \begin{algorithmic}[1]
    \State $X_1 =Y_1= P_{\Omega}(Q_2)$, $t=1$
    \While{ $|\sigma_1(X_{t-1}) - \sigma_1(X_{t-2})| > \epsilon$ }

      \State $\sigma_1(X_t), u_1 \leftarrow SVD(Y_t)$
 \State $Y_t \leftarrow \lambda u_1u_1^T - P_{\Omega}(\lambda u_1u_1^T - Q_2)=  Q_2 - diag(Q_2) + diag(\lambda u_1u_1^T)$
\State $X_{t+1}\leftarrow \sigma_1(X_t)u_1u_1^T$
    \EndWhile
    \State \Return $[\lambda=\sigma_1(X_t), v=u_1]$
  \end{algorithmic}
\end{algorithm}
In our proposed algorithm, we only need to calculate the first singular value and the corresponding singular vector at each iteration. The PROPACK Matlab package \citep{larsen2004propack} efficiently calculates the dominant singular values which will speed up our calculations especially for problems with large number of classifiers.

\section{Estimation of AUROC of the Base Classifiers using the Third Order Covariance Tensor}

The results of the previous sections help us estimate the vector $v$ whose entries are proportional to the performance of individual methods.  Moreover, from \eqref{eq:MLE1} we see that the vector $v$ was sufficient to form the SUMMA ensemble.  However, it is not possible to estimate the actual performance, i.e. the AUROC for each method, from the covariance matrix without {\it a priori} knowledge of the sample class prevalences.  To address this shortcoming, we developed a strategy for estimating the prevalence from the third order covariance tensor of unlabeled rank data. While our approach is similar to that of \citep{jaffe2015estimating}, we uniquely extend the iterative method presented in the previous section by using the generalization of singular value decomposition to tensor decomposition, \citep{karami2012compression}.
Given three methods $i,j$ and $k$, the third order covariance tensor is defined as
\begin{equation}
	Q_3(i,j,k)=\avgb{\Big(r_i-\avg{r_i}\Big)\Big(r_j-\avg{r_j}\Big)
	\Big(r_k-\avg{r_k}\Big)}. \label{def:tensor}
\end{equation}
In the previous sections we assumed that the rank predictions by base classifiers were conditionally independent.  Extending this assumption to triplets, $f_i, f_j, f_l$  for $i \neq j \neq l$,
\begin{align}
  \prob{f_i(x) = r, f_j(x) = s, f_l(x) = t | \sigma_k} &=
    P(f_i(x) = r | \sigma_k)
    P(f_j(x) = s | \sigma_k)
    P(f_l(x) = t | \sigma_k)\label{def:inde3},
\end{align}
for the $k^{th}$ sample $\sigma_k \in \{0,1\}$.
Using this observation we present the following corollary of Theorem~\ref{thm:indep}.
\begin{theorem} \label{thm:prevelance}
Under the assumption of conditionally independent base classifiers, the elements $i \neq j \neq l$ from the third order covariance tensor, Def~\eqref{def:tensor}, are given by,
\begin{equation}
	Q_3(i, j, l)=\rho (1-\rho) (1-2\rho) \Delta_i \Delta_j \Delta_l. \label{offdiag:tensor}
\end{equation}
\end{theorem}
\begin{proof}
	Let $i$, $j$ and $l$ be integers in the set $\set{1, 2, \dots, M}$ such that $i \neq j \neq l$.  The theorem is proved by applying Theorem~\ref{thm:indep} for $l=3$.
\end{proof}
Similar to the previous section, Theorem~\ref{thm:prevelance} shows that the tensor $T$ is an approximately rank-1 tensor with off-diagonal entries of $T$ given by $T \approx a \otimes a \otimes a$, where $a=\Big((\rho)(1-\rho)(1-2\rho)\Big)^{1/3} \Delta$.  In order to estimate $\Delta$ from the third order statistics, we extend our iterative method for decomposing the covariance matrix by using tensor SVD (tSVD) (\citep{SIAMRev2009Kolda}). Pseudo-code for the generalized iterative method algorithm is as follows:
\begin{algorithm}[H]
  \caption{Find rank 1 matrix from off-diagonal observations of covariance tensor}\label{alg:iter_t}
  \begin{algorithmic}[1]
    \State $X_1 =Y_1= P_{\Omega}(Q_3)$, $t=1$
    \While{ $|\sigma_1(X_{t-1}) - \sigma_1(X_{t-2})| > \epsilon$ }

      \State $\sigma_1(X_t), u_1 \leftarrow tSVD(Y_t)$
 \State $Y_t \leftarrow \lambda u_1u_1^T - P_{\Omega}(\lambda u_1u_1^T - Q_3)=  Q_3 - diag(Q_3) + diag(\lambda u_1u_1^T)$
\State $X_{t+1}\leftarrow \sigma_1(X_t)u_1u_1^T$
    \EndWhile
    \State \Return $[\lambda=\sigma_1(X_t), v=u_1]$
  \end{algorithmic}
\end{algorithm}
Similar to the covariance matrix case, the above algorithm gives us an estimate of the singular value of the tensor
\begin{equation}
	\lambda_t=\Big((\rho)(1-\rho)(1-2\rho)\Big) (\norm{\Delta})^3. \label{eq:t3}
\end{equation}
Lets refer to the eigenvalue corresponding to the rank one matrix $R$ estimated from $Q_2$ in Equation~\eqref{eq:lambda} as $\lambda_e$, and recall that it is equivalent to $\lambda_e = \rho(1-\rho) (\norm{\Delta})^2$.  Then, by defining the ratio $\beta=\lambda_t^2 / \lambda_e^3$ using Equations~(\ref{eq:t3},\ref{eq:lambda}) we find that,
\begin{eqnarray}
	\beta &=& (1-2\rho)^2/(\rho(1-\rho))\nonumber\\
	&\implies& \rho^2(\beta+4)-(4+\beta)\rho+1=0.\label{eq:solve_beta}
\end{eqnarray}
By rearranging the terms in Equation \eqref{eq:solve_beta} we find that
\begin{align}
	\rho (1-\rho)=\frac{1}{\beta+4} \implies
	\norm{\Delta}=\sqrt{\lambda_e	(\beta+4)}.
\end{align}
Knowledge of $\norm{\Delta}$ and by Equation \eqref{eq:lambda} we can compute the prevalence of each class label and for each $i$ base classifiers compute the true value $\Delta_i$.  Furthermore, by Theorem~\ref{thm:auc}, we may use $\Delta_i$ to compute the AUROC for each $i$ method.  It is clear now that under the assumption of conditional independence and weakly predictive base classifiers, we may estimate the AUROC for each base classifier without the knowledge of the underlying sample class labels.

\section{Numerical Examples}
%
In this section, we assess the performance of the SUMMA ensemble on i) synthetic data, and ii) real world data sets from the UCI Machine Learning Repository (\citealt{resource2013Lichman}).  In all cases we compare the performance of our SUMMA ensemble to the best individual base classifier and the unweighted average of base predictions denoted as the Wisdom of the Crowds ensemble, WOC for short (\citealt{NatMeth2012Marbach}).

\subsection{Synthetic Data}
%
In this section we validate the SUMMA ensemble using synthetic data.  Each data set represents $N$ sample rank predictions from $M$ conditionally independent base classifiers with unique performances.  Synthetic predictions were generated by producing random scores from two Gaussian distributions: one Gaussian represented scores from the negative class, while the other scores from the positive class.  A specific AUROC of each base classifier was controlled by adjusting the parameters of the respective class specific Gaussian distributions.  In this strategy the conditional independence assumption was satisfied by independently sampling the Gaussian distributions associated with each base classifier.  Once the samples were generated, we converted the scores to sample ranks.

We generated synthetic predictions using the outlined strategy for $M=30$ base classifiers, 500 samples from the positive and negative class for a total of $N=1000$ samples. We adjusted the parameters of the Gaussian distributions such that the distribution of base classifier AUROC values was uniformly distributed between $(0.4,0.8)$. Using these data we found SUMMA correctly estimated each base classifier's performance with a coefficient of determination of 0.95 (Figure~\ref{fig:1}a). In addition, Figure~\ref{fig:1}a shows that the SUMMA ensemble (red, AUROC=0.95) out-performs the best individual method, and the WOC ensemble (blue, AUROC=0.89), which aggregates prediction by averaging the sample ranks of the base classifiers.
\begin{figure}
	\centering
	\begin{subfigure}[t]{0.475\textwidth}
		\centering
		\includegraphics[width=\textwidth]{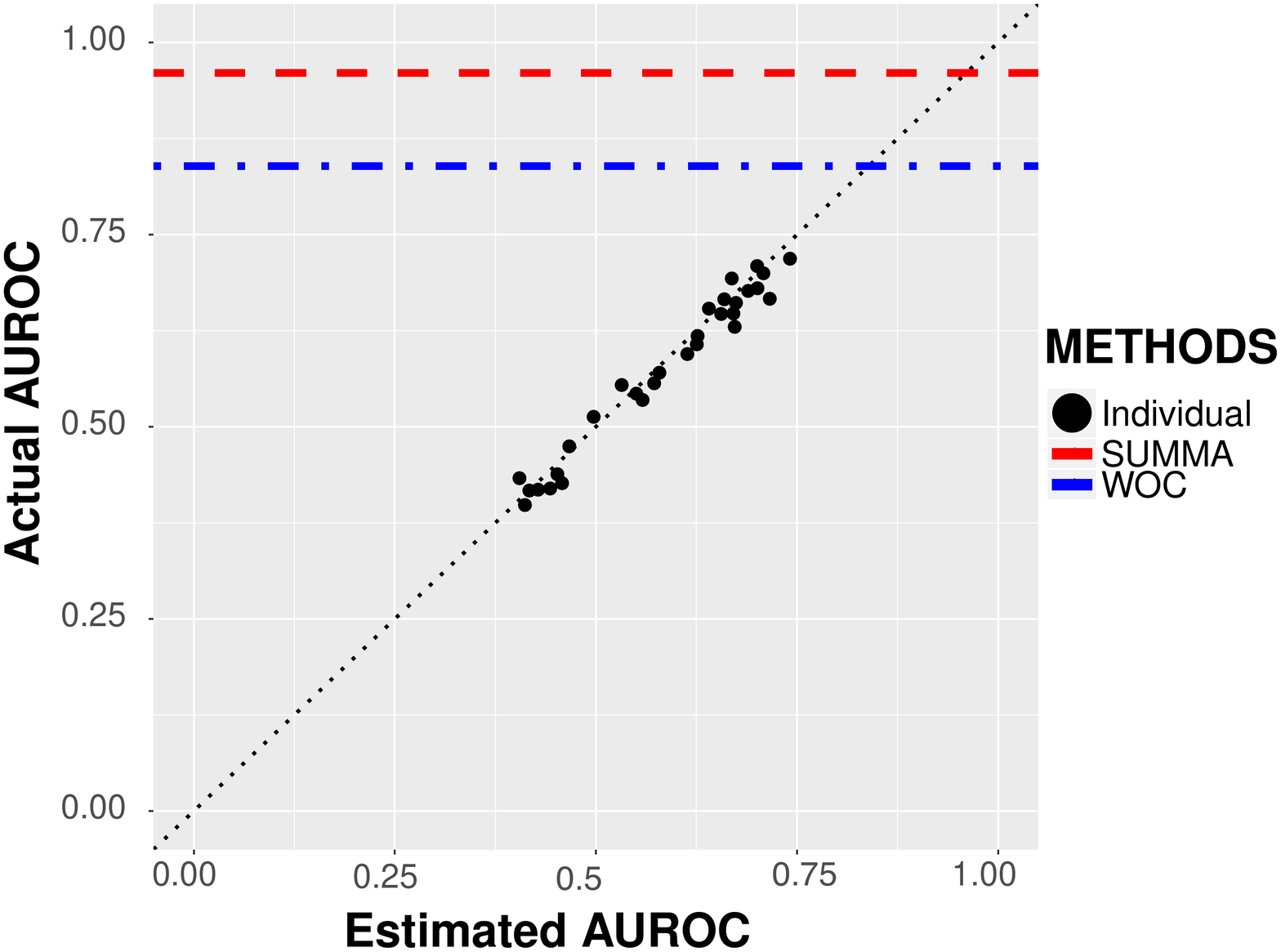}
		\caption{Inference of AUROC}
	\end{subfigure}
	\begin{subfigure}[t]{0.475\textwidth}
		\centering
		\includegraphics[width=\textwidth]{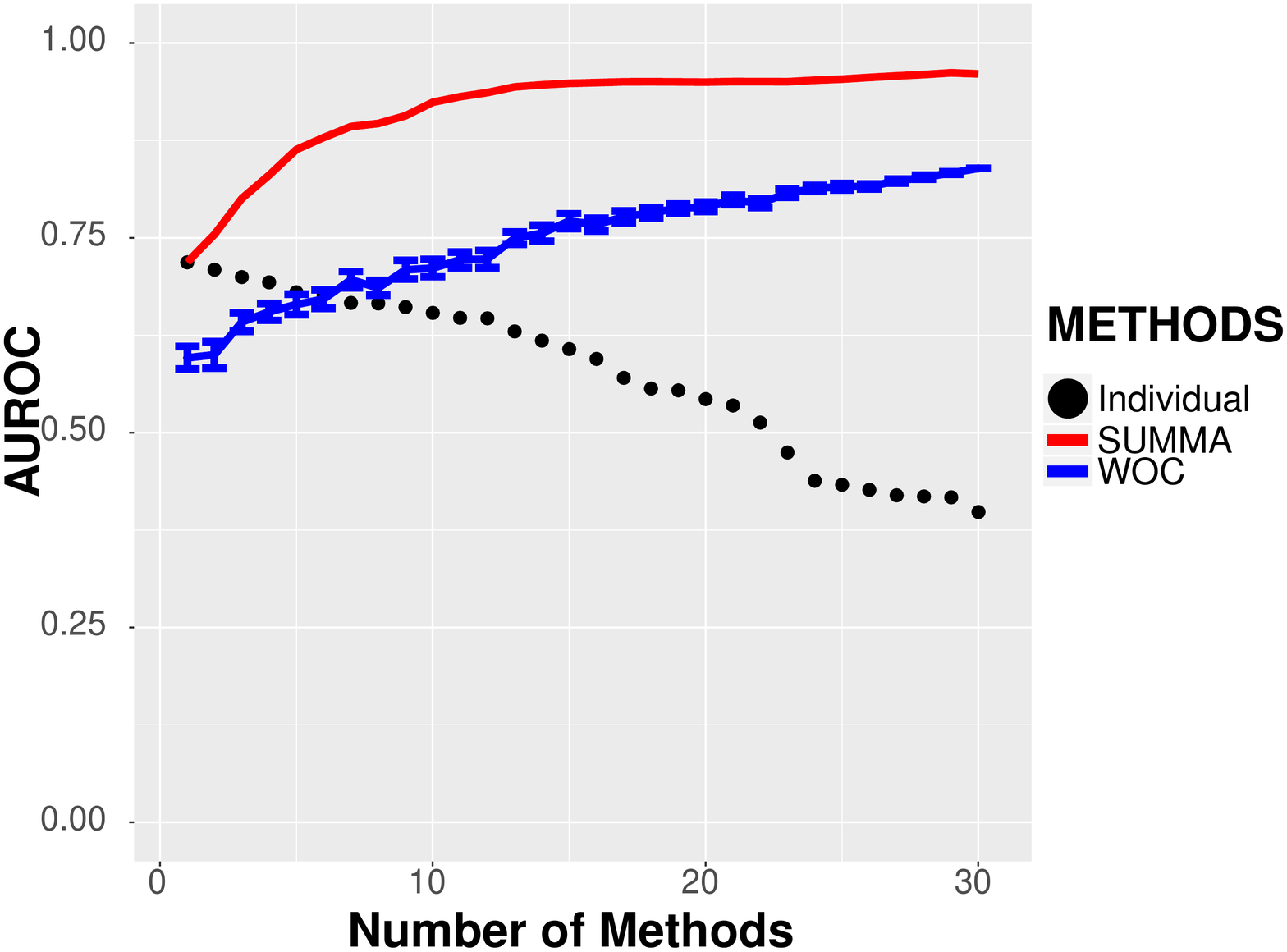}
		\caption{SUMMA ensemble and the number of base classifiers}
	\end{subfigure}
	\caption{Validation and analysis of the SUMMA ensemble with simulation data.}
	\label{fig:1}
\end{figure}

Figure~\ref{fig:1}b shows how the performance of the SUMMA and WOC ensembles change with the number of base classifiers. Here, we construct the $n^{th}$ SUMMA ensemble by aggregating the top $n$ performing methods.  Unlike SUMMA, in this unsupervised setting the WOC ensemble has no knowledge of the performances of each base classifier.  As a consequence, to predict the AUROC of n classifiers in Fig 1b the WOC ensemble chooses $n$ base classifiers at random. To achieve robust performance estimates for the $n^{th}$ WOC ensemble we measured the mean and standard error of the mean AUROC for 50 replicate WOC ensembles constructed from $n$ base classifiers.  Empirically, we find that the SUMMA ensemble increases in performance more readily as methods are added and saturates at a higher AUROC than the WOC ensemble (Figure~\ref{fig:1}b).  Of note, is that the SUMMA ensemble saturates, which suggests that a subset of base classifiers is sufficient to achieve the maximum performance. We leave investigation of this observation to a future study.

\begin{figure*}[htbp]
    \centering
    \begin{subfigure}[t]{0.47\textwidth}
      \includegraphics[width=3in]{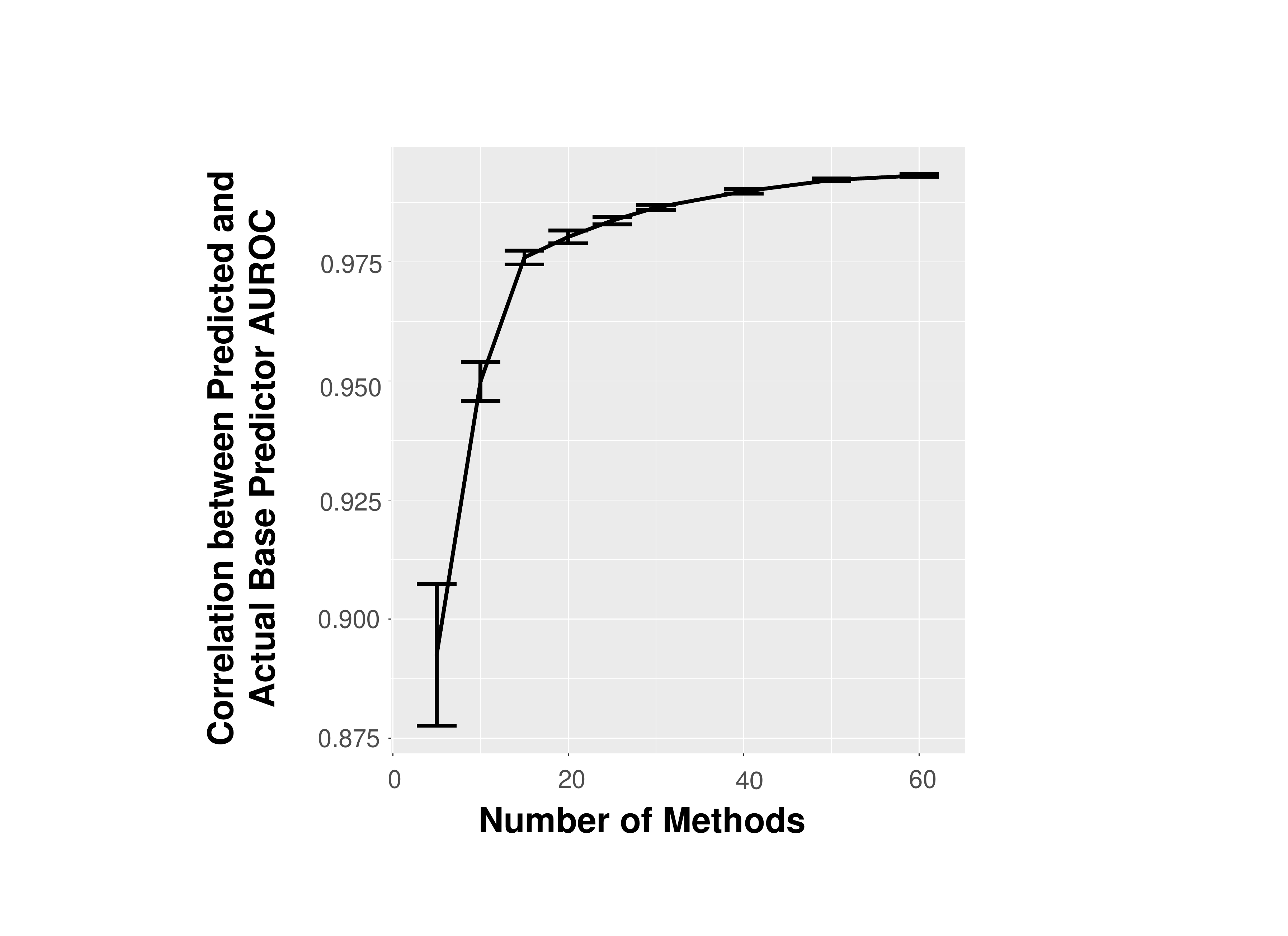}
        \caption{Correlation of true and SUMMA inferred AUROC of each base classifiers}
    \end{subfigure}%
    \begin{subfigure}[t]{0.47\textwidth}
        \includegraphics[width=3in]{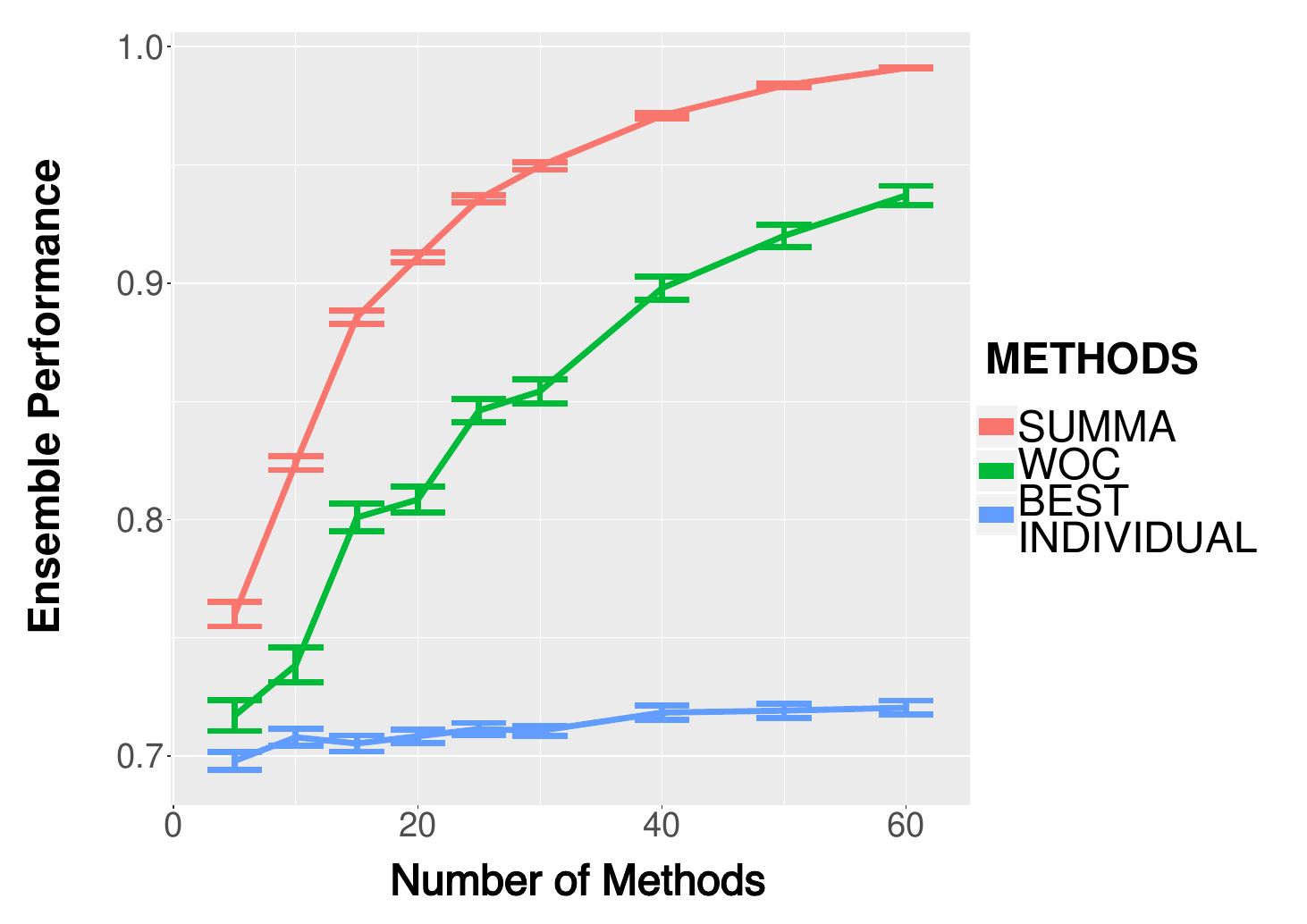}
        \caption{AUROC and the number of base classifiers}
    \end{subfigure}
    \caption{The dependence of the SUMMA ensemble with the number of base classifiers for $N=1000$ samples and $\rho=0.5$.}
\label{fig:aurocMclassifiers}
\end{figure*}

Next we empirically test the dependence of the SUMMA on varying number of methods, samples, and class prevalences.  In each case we change a single simulation parameter from the default values of, $M=30$ base classifiers, $N=1000$ samples, and the prevalence of class 1 $\rho = 1/2$.

We tested the influence of the number of methods by simulating predictions and applying the SUMMA algorithm to ensembles composed of $M = \set{5, 6, 7, \dots, 30}$ base classifiers.  This experiment is different from the analysis of Fig 1.b, because the SUMMA algorithm is applied to $M=\set{5, 6, 7, \dots, 30}$ (Fig. 2a) base classifiers as opposed to the covariance matrix of all 30 base classifiers.  Intuitively, inferring $\Delta$ for larger covariance matrices should become more accurate because the number of equations grows faster, $M(M-1)/2$, than the number of parameters $\Delta_i$, $M$.  Indeed, this intuition is confirmed in Figure~\ref{fig:aurocMclassifiers}. We see that the correlation between the predicted versus actual AUROC of base classifiers inferred by SUMMA for $M=5$ is $\approx$0.875, and increases readily to $\approx$0.975 for $M \geq 15$. The error bars is the figure is the standard error of the mean for 30 repeated experiments. We then tested how the number of methods affected the performance of the corresponding SUMMA ensemble.  In Figure~\ref{fig:aurocMclassifiers}b we find that the SUMMA ensemble outperforms the WOC and best individual performing methods. Moreover, as the number of methods increases the SUMMA performance converges to the perfect $1$.

From the Central Limit Theorem, one would expect that the performance of the SUMMA algorithm would increase with the number of samples.  This is simply because the error in estimating the covariance matrix elements decreases with $N$.  Indeed, we find that the accuracy of the SUMMA inferred AUROC of each respective method monotonically increases with a correlation of $\approx$0.575 to $\approx$1 with increasing $N$ from 30 to 4000 samples, Figure~\ref{fig:aurocNsamples}a.  Furthermore, the uncertainty of the covariance elements in data sets with less than 250 samples is sufficiently large to negatively influence the SUMMA ensemble performance (Figure~\ref{fig:aurocNsamples}b).

\begin{figure*}[t!]
    \centering
    \begin{subfigure}[t]{0.47\textwidth}
        \centering
        \includegraphics[width=3in]{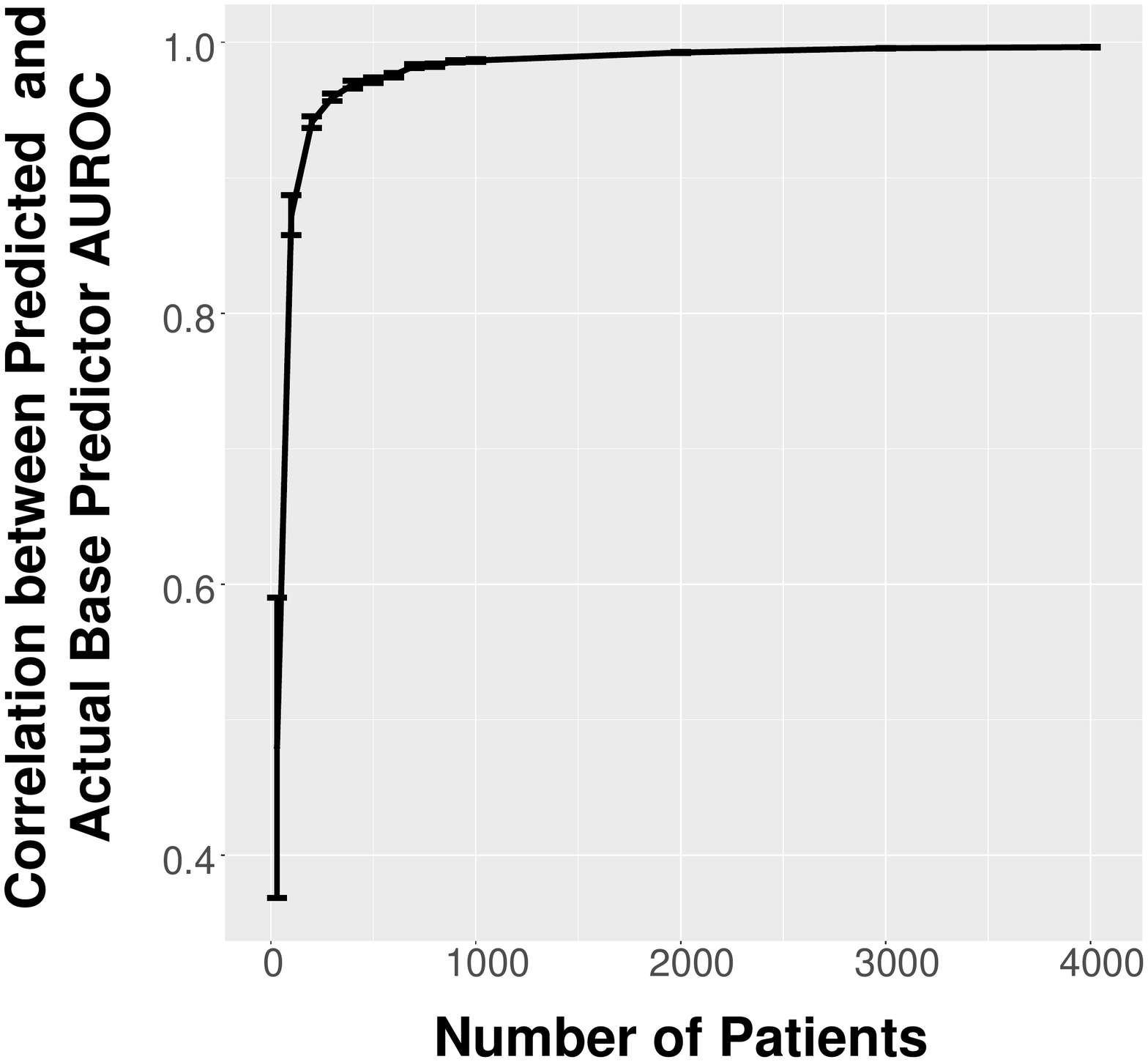}
        \caption{Correlation between the true and SUMMA inferred AUROC of each base classifiers}
    \end{subfigure}%
    ~
    \begin{subfigure}[t]{0.47\textwidth}
        \centering
        \includegraphics[width=3in]{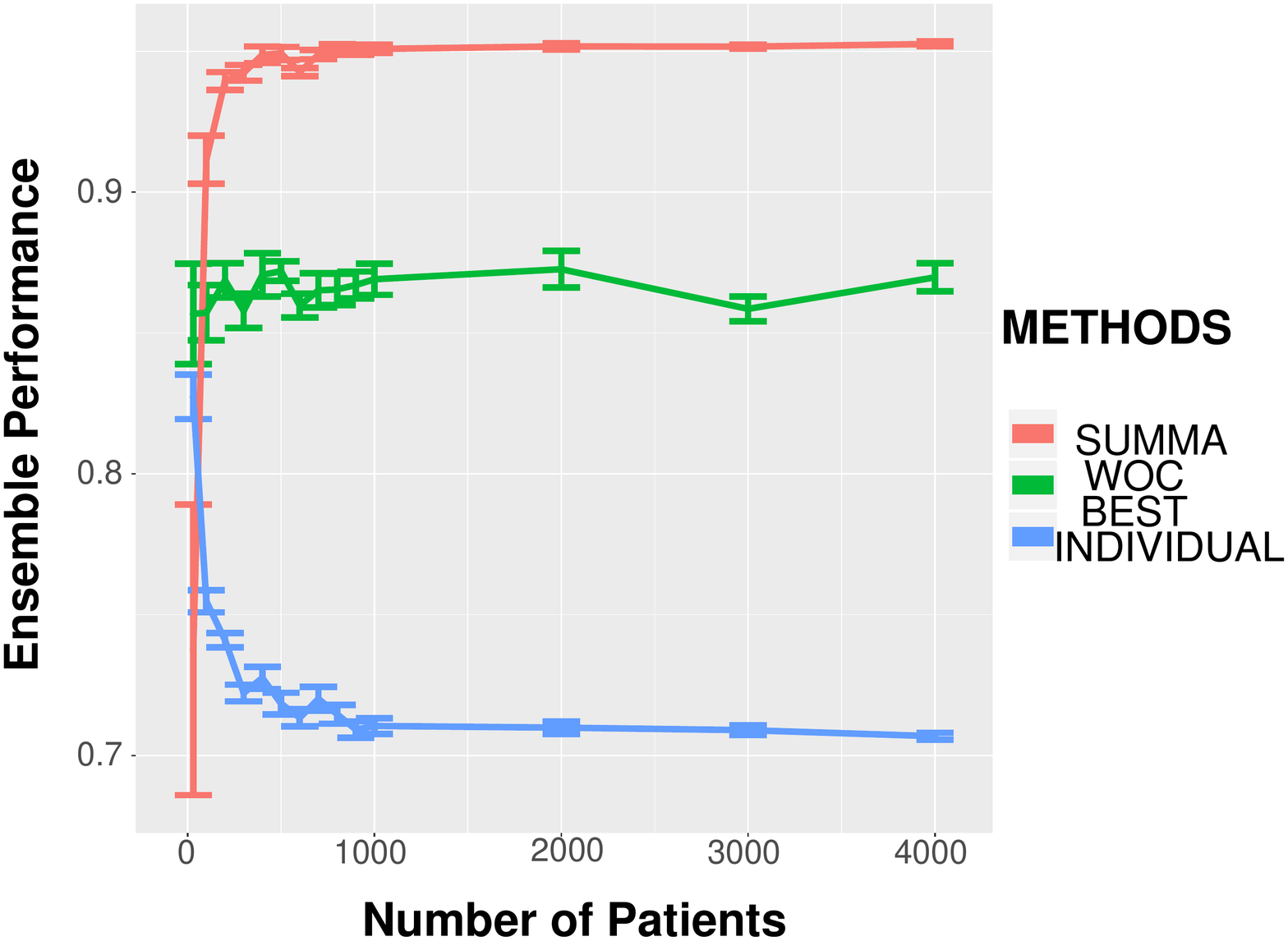}
        \caption{AUROC and the number of samples}
    \end{subfigure}
    \caption{The dependence of the SUMMA ensemble with the number of samples}
		\label{fig:aurocNsamples}
\end{figure*}

Lastly, we tested the influence of the class prevalence on the SUMMA algorithm.  We find that the accuracy of the SUMMA inferred AUROC decreases for $\rho\leq$0.2 or $\rho\geq$0.8, Figure~\ref{fig:aurocPrevalence}a.  We found this to be an intuitive result, because as the class prevalence moves to the extremes the dominating signal originates from a single sample class.  In these simulations the distribution of sample rank predictions by base classifiers are conditionally independent.  Consequently, as the prevalence goes to zero or one the empirical covariance should be smaller than the theoretical covariance.  Moreover, we find that the corresponding SUMMA ensemble from highly imbalanced data under-performs that from balanced data, albeit, not by much as shown in Figure~\ref{fig:aurocPrevalence}b.

\begin{figure*}[t!]
    \centering
    \begin{subfigure}[c]{0.47\textwidth}
        \centering
        \includegraphics[width=2in]{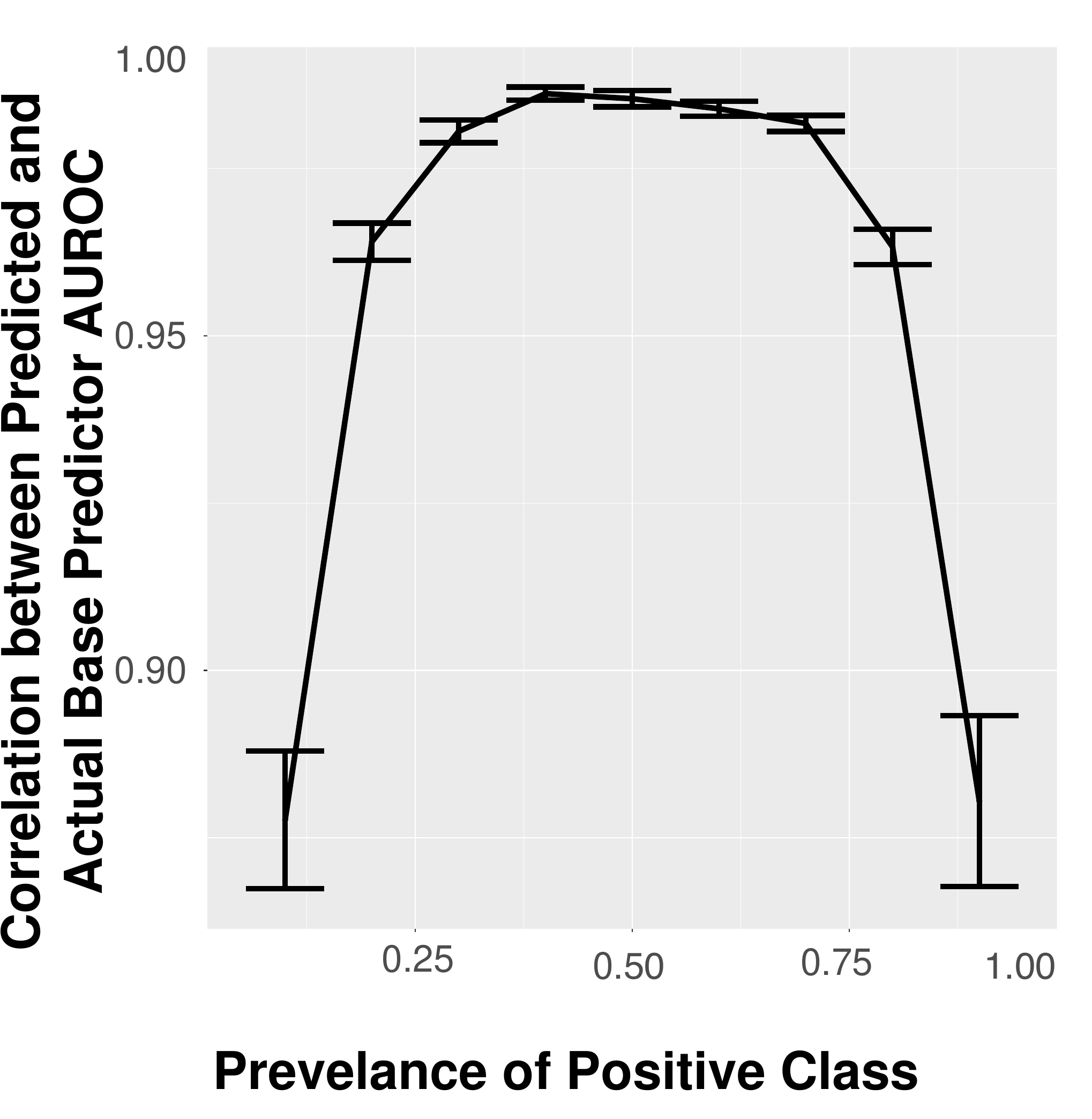}
        \caption{Correlation between true and SUMMA inferred AUROC of each base classifier}
    \end{subfigure}%
    ~
    \begin{subfigure}[c]{0.47\textwidth}
        \centering
        \includegraphics[width=3in]{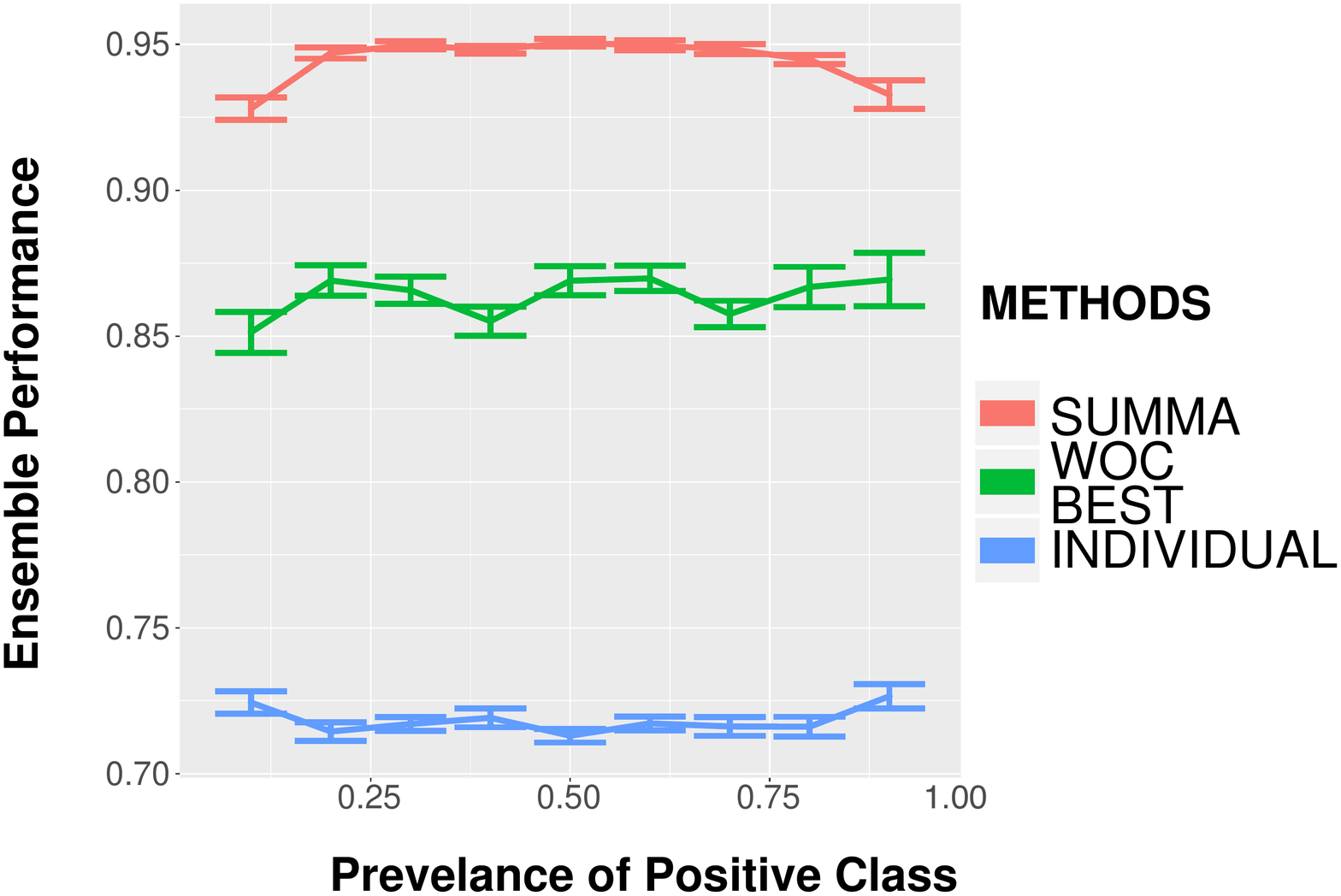}
        \caption{Ensemble AUROC and the class prevalence}
    \end{subfigure}
    \caption{The dependence of the SUMMA ensemble with the prevalence of class 1}
		\label{fig:aurocPrevalence}
\end{figure*}

\subsection{Real World Data from the UCI Machine Learning Repository}

In this subsection we test the SUMMA algorithm on six data sets (Table~\ref{t:1}) from the UCI Machine Learning Repository \citep*{blake1998uci, resource2013Lichman}.
\begin{table}
	\begin{center}
    \resizebox{0.98\columnwidth}{!}{
		 \begin{tabular}{||c | c  | c | c | c  ||}
		 \hline
		 Name & \# Features & \# Samples & Reference & Prevalence of \\ & & & & Minority Class\\ [0.5ex]
		 \hline
		Bank Marketing & 17 & 45211 & \citep{DSS2014Moro} & 0.11\\
		 \hline
		  &  &   & \citep{LSNO1990Mangasarian} &\\
		 Breast Cancer (Wisconsin) & 10 & 683 & \citep{PNAS1990Wolberg} & 0.34\\
		  &  &   & \citep{OMS1992Bennett} &\\
		 \hline
		 Ionosphere & 35 & 351  & \citep{resource2013Lichman} & 0.35\\
		 \hline
		 Mammographic Mass & 6 & 830 &  \citep{MedicalPhys2007Elter}& 0.48\\ [1ex] 
		 \hline
		Parkinsons & 23 & 195 &  \citep{BEO2007Little} &0.24\\ [1ex]
		 \hline
		Yeast & 9 & 892 &  \citep{resource2013Lichman} & 0.48\\ [1ex]
		 \hline
		\end{tabular}
        }
		\caption{Summary of Real World Data Sets from UCI Machine Learning Repository}
	 \label{t:1}
	\end{center}
\end{table}

First we developed SUMMA ensembles by training base classifiers from the R package caret on each data set in Table~\ref{t:1}.  We choose this package for its ease of use, the large diversity of methods, $M=22$ (Table~\ref{t:2}), and tools for cross-validation \citep{JSS2008Kuhn}.  Due to the diversity of the data sets, we developed two strategies for training the base classifiers.  First, for all the data sets with less than 2000 samples we randomly divided the samples into two sets, one for training and the other for testing.  Second, for the Bank Marketing Data Set, which had 45,000 samples, we randomly selected 1000 samples for training and different 1000 samples for the test set.  We then used the test set AUROC for testing the SUMMA inference strategy.

\begin{table}
	\begin{center}
		 \begin{tabular}{||c | c  | c ||}
		 \hline
		 Name & Main Method & RLibrary\\ [0.5ex]
		 \hline
		adaboost & Adaboost & fastAdaboost  \\
		 \hline
		avNNet& Model Averaged Neural Network & nnet  \\
		 \hline
		 bayesglm& Bayesian Generalized Linear Model & arm  \\
		 \hline
		 ctree & Conditional Inference Tree & party  \\
		 \hline
		 earth & Multivariate Adaptive Regression Spline & earth  \\
		 \hline
		  gbm & Stochastic Gradient Boosting & gbm  \\
		 \hline
		 glm & Generalized Linear Model & stats  \\
		 \hline
		 glmnet & Lasso and Elastic-Net Regularized Generalized Linear Models & glmnet  \\
		 \hline
		 J48 & C4.5-like Trees & RWeka  \\
		 \hline
		 Jrip & Rule-Based Classifier & RWeka  \\
		 \hline
		  C5.0 & Decision Trees and Rule-Based Models & C50 \\
		 \hline
		knn & k-Nearest Neighbors& kknn  \\
		 \hline
		  LMT & Logistic Model Trees & RWeka  \\
		 \hline
		 mlp& Multi-Layer Perceptron & RSNNS  \\
		 \hline
		 nb & Naive Bayes & klaR  \\
		 \hline
		  nnet & Neural network & nnet  \\
		 \hline
		  rf & Random Forest & randomForest  \\
		 \hline
		  rpart & Recursive Partitioning and Regression Trees & rpart  \\
		 \hline
		 simpls & Partial Least Squares & pls  \\
		 \hline
		svmLinear2& Support Vector Machine with Linear Kernel & e1071  \\
		 \hline
		 svmRadial & Support Vector Machine with Radial Kernel  & kernlab  \\
		 \hline
		xgbLinear & eXtreme Gradient Boosting & xgboost \\
		 \hline
		  xgbTree & eXtreme Gradient Boosting & xgboost  \\
		 \hline
		\end{tabular}
		\caption{Machine Learning Methods Used}
		\label{t:2}
	\end{center}
\end{table}
To test whether the SUMMA algorithm can infer the AUROC of each base classifier trained on the six UCI Machine Learning Repository data sets.  In all the data sets the correlation between the true and SUMMA inferred AUROC for all algorithms is above 0.75, Figure~\ref{fig:ex2_cor}.  The lower correlation in AUROC values from real world data as opposed to the synthetic data (correlation = 0.98) is likely a result of conditional dependence between base classifiers. Although this might lead to the decrease in the performance of SUMMA, Table~\ref{t:3} shows the rankings of top methods of each data set for the remaining datasets. We can easily observe that methods that perform best in one data set do not necessarily perform well in other data sets.  In fact, they can be one of the worst in other data sets.  This is most likely due to the distributions of the data being different in different data sets and methods with different theoretical backgrounds are more suitable to be applied in one type of data than other.

In comparison,  SUMMA performs better than the best individual method in the Bankmarketing, Parkinsons and Yeast datasets, and second best in the mammographic masses and ionosphere data sets.  For the breast cancer data set, 18 methods including SUMMA have an AUC of 0.97 and 7 other methods have performed better than SUMMA.  However, since almost every method has performed almost perfectly this indicates that the methods are more likely to be correlated.  Finally, for two of the datasets, Bankmarketing and Yeast, we analyzed how the number of integrated methods affects the performance of SUMMA prediction by examining randomly sampled combinations of individual methods ( Figures~\ref{fig:BMandYeast}a and b respectively).  SUMMA performs better than individual inference methods even when integrating small sets of individual predictions.  Performance increases further with the number of integrated methods. For instance, for 15 randomly selected inference methods, the SUMMA ensemble performs better than the best amongst the 15 methods in 98\% of the cases in the bankmarketing data set and it ranks best in 80\% of the cases and  best or second best in 97\% of the cases in the yeast data set demonstrating the robustness of SUMMA. Table~\ref{t:4} shows the frequency with which SUMMA outperforms the WOC ensemble prediction in the bankmarketting and yeast data. For example, for the bankmarketing data set if we combine random 10 teams, 99\% of the times SUMMA performs better than WOC. For the yeast data, SUMMA gives a better prediction in about 65\% of the times.

\begin{figure*}[t!]
  \centering
  \includegraphics[height=4in]{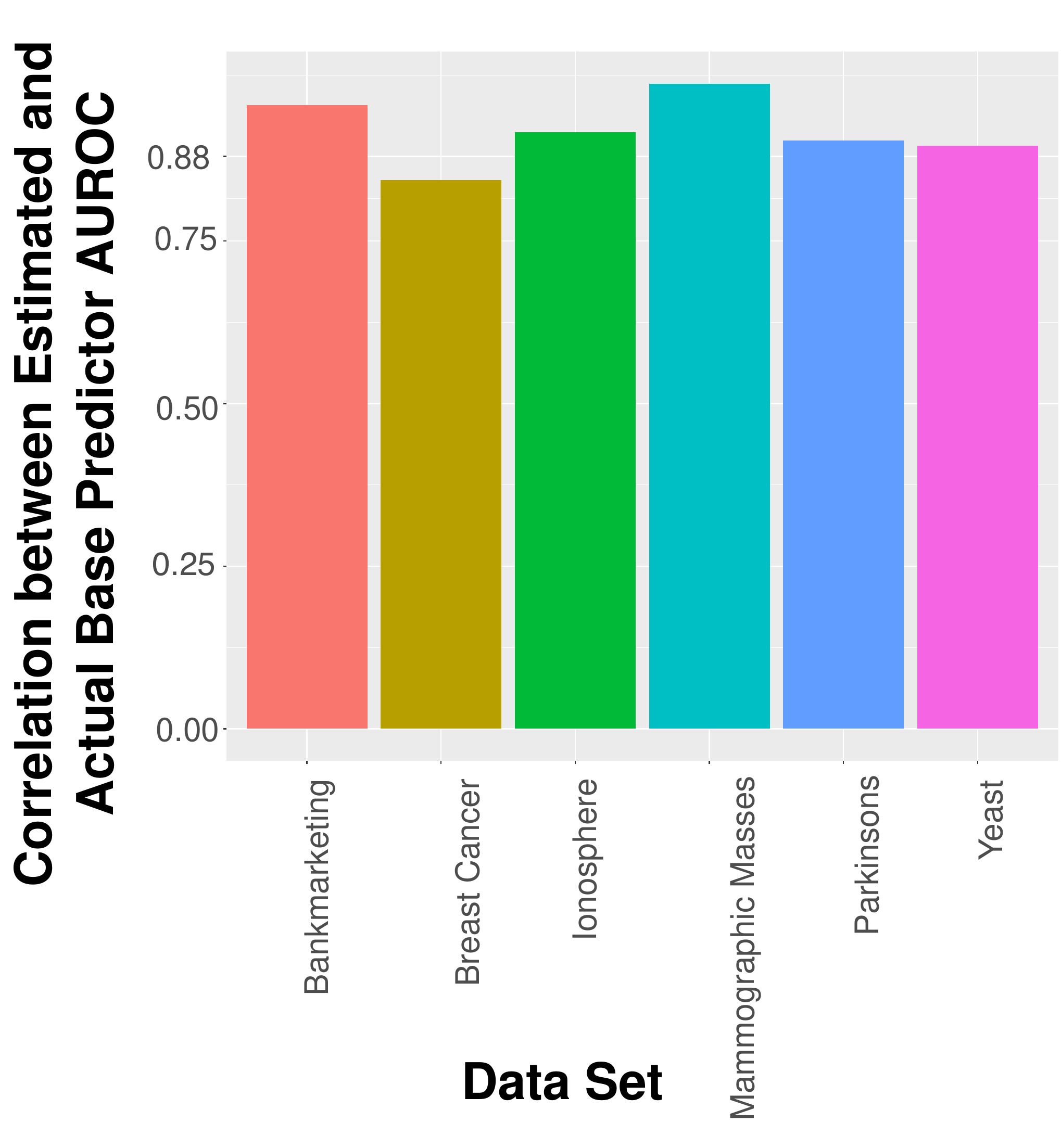}
  \caption{Correlation between the actual and estimated base AUROC on the 6 UCI dataset}
  \label{fig:ex2_cor}
\end{figure*}


\begin{figure*}[t!]
    \centering
    \begin{subfigure}[t]{0.5\textwidth}
        \centering
        \includegraphics[height=2in]{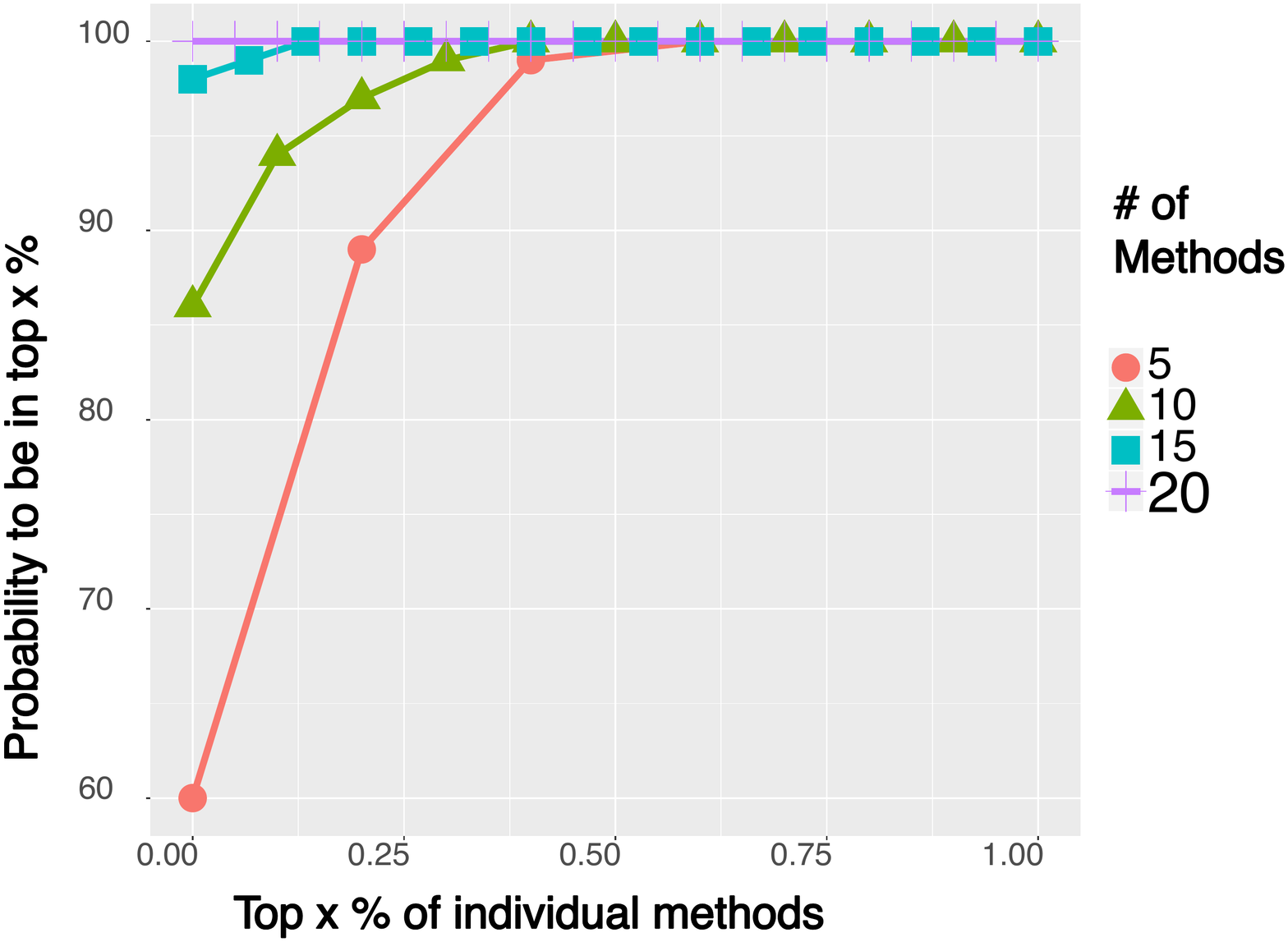}
        \caption{Bankmarketing Data}
    \end{subfigure}%
    ~
    \begin{subfigure}[t]{0.5\textwidth}
        \centering
        \includegraphics[height=2in]{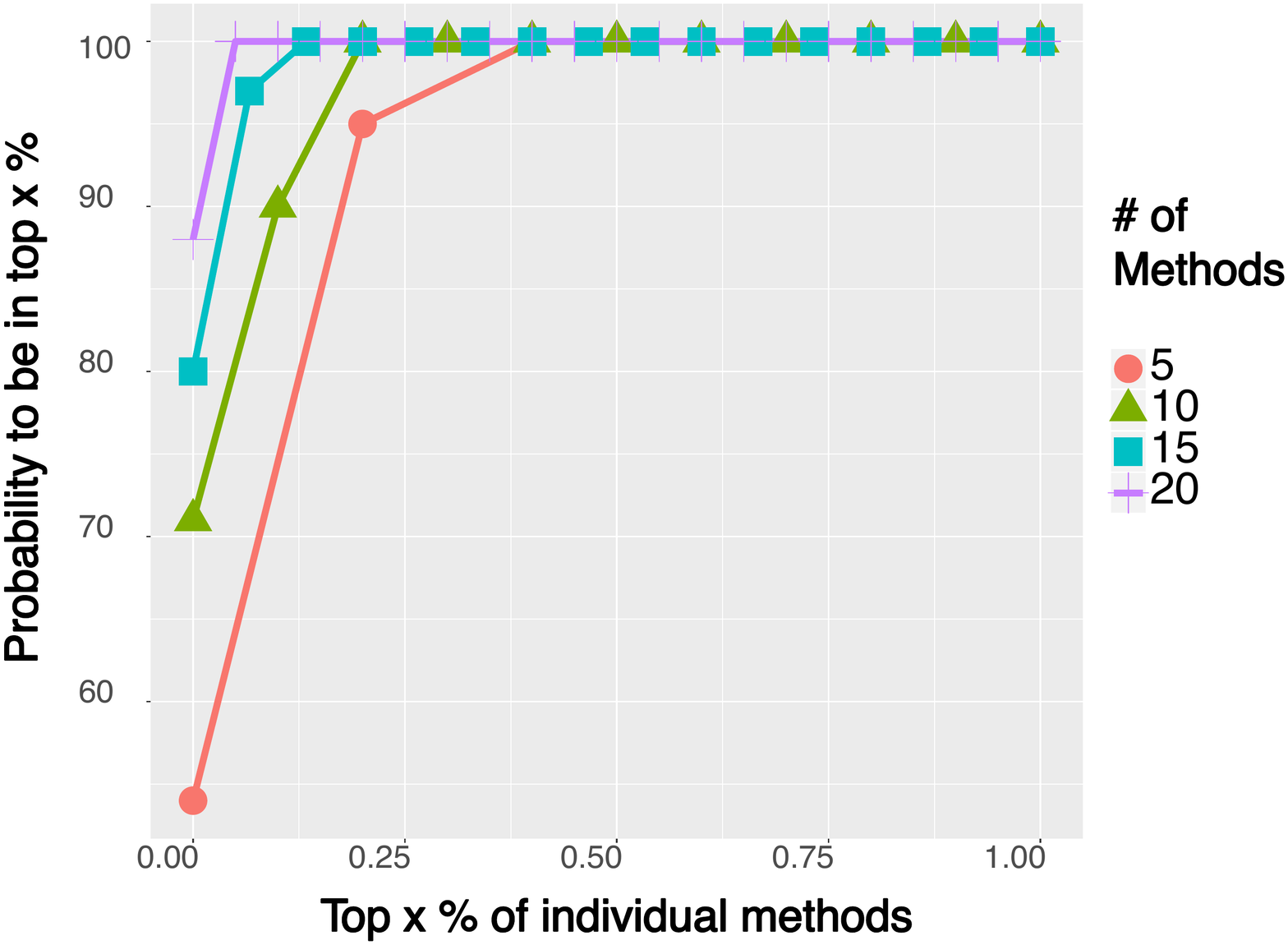}
        \caption{Yeast Data}
    \end{subfigure}
    \caption{SUMMA performance with increasing number of methods}
		\label{fig:BMandYeast}
\end{figure*}


\begin{table}
\begin{center}
	\begin{tabular}{||c | c| c | c | c | c |c ||}
	 \hline
	 	 & Bank & Breast &  & Mammographic &  & \\[0.5ex]
	 	Method & Marketing & Cancer & Ionosphere & Mass &  Parkinsons & Yeast\\[0.5ex]
	 \hline
		Earth & 2 &19&22 &5 &6 &6 \\
	 \hline
		Simpls & 14 & 1 & 15 &17 &12 &20\\
	 \hline
	 	svmRadial & 12 & 18 &1 &8 &5 &18\\
	 \hline
	 	gbm & 5 & 15 & 7  & 1 &4 &4  \\
	 \hline
		C5.0& 8 & 10 & 8 & 7 &21 &2  \\ [1ex]
	 \hline
		rf  & 7 & 9 &3 &6 &1& 5  \\
	 \hline
     SUMMA & 1 & 8 & 5 & 2 & 1 & 1 \\ \hline
	\end{tabular}
	\caption{Method Ranking (The Best Base Classifier in Each Data Set is Listed)}
 	\label{t:3}
\end{center}
\end{table}


\begin{table}
\begin{center}
 	\begin{tabular}{||c| c| c||}
    \hline
        & \multicolumn{2}{c||}{\% SUMMA is better than WOC}\\
	 	Number Methods & Bank Marketing Data & Yeast Data\\ [0.5ex]
        \hline
		5 & 90 &69\\
	 	\hline
		10 & 99 & 68\\
	 	\hline
		15 & 100 & 63\\
	 	\hline
		20 & 100 & 68  \\
		\hline
	\end{tabular}
	\caption{Percentage of times that SUMMA outperforms WOC.}
 \label{t:4}
\end{center}
\end{table}


\section{Conlusions}

In this paper, we introduced the Strategy for Unsupervised Multiple Method Aggregation, SUMMA.  We showed that by using the SUMMA algorithm we may infer the AUROC of conditionally independent base classifiers from the covariance of their rank predictions in the absence labeled data.  These inferred AUROC values are then incorporated into the SUMMA ensemble classifier, whose empirical performance was tested on simulated data and real world data sets with commonly used base classifiers alike.

Our strategy is a generalization of the Spectral Meta Lerner (SML) method developed by \citep{parisi2014ranking}, where the authors use a binary prediction matrix for the analysis and the balanced accuracy as a performance measure. 
In the SML method of \citep{parisi2014ranking} , the authors starting point is the covariance matrix of binary predictions. They show that under the assumption of conditional independence of base predictors given the class, the covariance matrix has a rank 1 approximation whose eigenvector entries are proportional to the balanced accuracies of the base predictors. In the case of SUMMA, our starting point is the covariance matrix of conditionally independent ranked predictions. We showed that this matrix has a rank 1 approximation whose eigenvector entries are related to the AUROC of the base predictors. A binary classifier can be considered to be a degenerate version of a ranked prediction in which the ranks assigned to the predicted positive class are all equal and larger than the unique rank assigned to the predicted negative class. In this degenerate case, it can be shown that the AUROC of such binary classifier is equal to the balanced accuracy. 
However, their analysis discards some information that comes from the ranking of samples that is readily available with most of the classifier used in everyday applications.

Our method SUMMA is well suited to be applied to methods generated in data challenges such as DREAM, Kaggle, KDD, etc. In such challenges the participants either use various modifications of existing methods or generate their own methods as such it is most likely that the methods are nearly conditionally independent. Moreover, our simulation results show that the best method in one problem is not necessarily the best in other problems. SUMMA, on the other hand, performed better than any other individual method on average. Therefore, SUMMA lends itself as a methodology that integrates different algorithms with a robust performance across different problems which is a desirable feature in community challenges.  

It is clear from the simulation studies that after aggregating a minimum number of methods, the SUMMA performance is stabilized and new methods do not improve the ensemble performance. A theoretic analysis for a stopping condition is left for future study.
A natural extension of our work would be to build a framework where there exists some structured correlation in the data.  An example could be a block structure where a group of methods are correlated with each other but uncorrelated with predictors in other groups.  A first step in such a problem would be to identify the block structure and then average the predictions in each block before running SUMMA.  Another extension of our work be a version of SUMMA for regression problems.  In fact, by requiring each method to rank samples we did a first step towards formulating the regression equivalent of SUMMA.  Further research is needed in this direction to build a theory.


\bibliography{jmlr}

\end{document}